\documentclass[nohyperref]{article}

\usepackage{microtype}
\usepackage{graphicx}
\usepackage{booktabs} 
\usepackage{hyperref}

\usepackage[accepted]{icml2022}

\usepackage{amsmath}
\usepackage{amssymb}
\usepackage{mathtools}
\usepackage{amsthm}

\usepackage[capitalize,noabbrev]{cleveref}

\theoremstyle{plain}

\usepackage{amsmath,amsfonts,bm, esint, amsthm}
\usepackage{subcaption}

\usepackage{bbm}
\usepackage{amssymb}
\usepackage[T1]{fontenc}
\usepackage{enumitem}
\usepackage{xcolor}         
\usepackage{graphicx}

\usepackage{mathtools}
\usepackage{BOONDOX-uprscr}

\usepackage{sidecap}

\usepackage{wrapfig}

\newcommand{\dmi}{$(\D, \M, \I)$ }

\definecolor{darkred_f}{RGB}{182, 85, 85}
\definecolor{darkblue_f}{RGB}{86, 116, 172}
\definecolor{darkorange_f}{RGB}{209, 136, 92}
\definecolor{darkgreen_f}{RGB}{106, 165, 110}

\definecolor{plot_blue}{RGB}{66, 153, 225}
\definecolor{plot_orange}{RGB}{237, 137, 54} 
\definecolor{plot_red}{RGB}{245, 101, 101}
\definecolor{plot_green}{RGB}{72, 187, 120}
\definecolor{plot_purple}{RGB}{159, 122, 234}
\definecolor{plot_green2}{RGB}{56, 178, 172}
\definecolor{plot_pink}{RGB}{237, 100, 166}
\definecolor{goldenyellow}{rgb}{1.0, 0.87, 0.0}

\def\Eqref#1{Equation~\ref{#1}}
\def\eqref#1{\Eqref{#1}}

\def\1{\bm{1}}

\DeclareMathAlphabet{\mathsfit}{\encodingdefault}{\sfdefault}{m}{sl}
\SetMathAlphabet{\mathsfit}{bold}{\encodingdefault}{\sfdefault}{bx}{n}

\def\gG{{\mathcal{G}}}

\def\gX{{\mathcal{X}}}

\newcommand{\E}{\mathbb{E}}

\newcommand{\R}{\mathbb{R}}

\usepackage{subcaption}
\usepackage{bbm}
\usepackage{mathtools}
\usepackage{amsthm, amssymb}
\newtheorem{theorem}{Theorem}[section]

\newtheorem{corollary}{Corollary}

\newtheorem{definition}{Definition}

\newcommand{\Odthree}{\text{O}(3d)}
\newcommand{\Othreed}{\text{O}(3)^d}
\newcommand{\Othree}{\text{O}(3)\otimes {\bf I}_d}

\def\1{\bm{1}}
\newcommand{\D}{\mathcal{D}}
\newcommand{\X}{\mathcal{X}}
\newcommand{\Y}{\mathcal{Y}}

\newcommand{\K}{\mathcal{K}}

\newcommand{\Id}{\bf Id}

\newcommand{\fcnn}{{\mathsf{FCN}}_n}
\newcommand{\fcnnn}{{\mathsf{FCN}}}

\newcommand{\ntk}{{\mathsf{NTK}}}

\newcommand{\REG}{{\mathsf{+L2}}}
\newcommand{\LR}{{\mathsf{+LR}}}

\newcommand{\fcn}{\mathsf {FCN}_\infty }
\newcommand{\lcnn}{\mathsf {LCN}_n}
\newcommand{\lcnnn}{\mathsf {LCN}}

\newcommand{\lcn}{\mathsf {LCN}_\infty }
\newcommand{\lcninf}{\mathsf {LCN}_\infty }

\newcommand{\vecn}{\mathsf {VEC}_n}
\newcommand{\vecnn}{\mathsf {VEC}}
\newcommand{\vecinf}{\mathsf {VEC}_\infty }

\newcommand{\lapn}{\mathsf {LAP}_n}

\newcommand{\lapinf}{\mathsf {LAP}_\infty }

\newcommand{\gapn}{\mathsf {GAP}_n}
\newcommand{\gapnn}{\mathsf {GAP}}
\newcommand{\gap}{\mathsf {GAP}_\infty }

\newcommand{\infntk}{\Theta}

\newcommand{\M}{\mathcal M}
\newcommand{\I}{\mathcal I}

\usepackage{hyperref}
\hypersetup{
    unicode=false,
    pdftoolbar=true,
    pdfmenubar=true,
    pdffitwindow=false,
    pdfstartview={FitH},
    pdfauthor={Anonymous},
    pdfkeywords={Infinite Neural Networks, Neural Tangent Kernel, Gaussian Processes, JAX, Empirical Study},
    pdfnewwindow=true,
    colorlinks=True,
    linkcolor=BurntOrange,
    citecolor=RawSienna,
    filecolor=magenta,
    urlcolor=blue
}

\usepackage{etoolbox,refcount}
\usepackage{multicol}

\newcounter{countitems}
\newcounter{nextitemizecount}
\newcommand{\setupcountitems}{%
  \stepcounter{nextitemizecount}%
  \setcounter{countitems}{0}%
  \preto\item{\stepcounter{countitems}}%
}
\makeatletter
\newcommand{\computecountitems}{%
  \edef\@currentlabel{\number\c@countitems}%
  \label{countitems@\number\numexpr\value{nextitemizecount}-1\relax}%
}
\newcommand{\nextitemizecount}{%
  \getrefnumber{countitems@\number\c@nextitemizecount}%
}
\newcommand{\previtemizecount}{%
  \getrefnumber{countitems@\number\numexpr\value{nextitemizecount}-1\relax}%
}
\makeatother    
{\end{itemize}%
\unskip\computecountitems\ifnumcomp{\previtemizecount}{>}{3}{\end{multicols}}{}}

\usepackage[textsize=tiny]{todonotes}

\hypersetup{
colorlinks,linkcolor={blue},citecolor={blue},urlcolor={blue}}

\icmltitlerunning{Synergy and Symmetry in Deep Learning}

\begin{document}

\twocolumn[
\icmltitle{
Synergy and Symmetry in Deep Learning:\\Interactions between the Data, Model, and Inference Algorithm}




\begin{icmlauthorlist}
\icmlauthor{Lechao Xiao}{yyy}
\icmlauthor{Jeffrey Pennington}{yyy}
\end{icmlauthorlist}

\icmlaffiliation{yyy}{Google Research, Brain Team}

\icmlcorrespondingauthor{Lechao Xiao}{xlc@google.com}
\icmlcorrespondingauthor{Jeffrey Pennington}{jpennin@google.com}

\icmlkeywords{Machine Learning, ICML}

\vskip 0.3in
]



\printAffiliationsAndNotice{}  

\begin{abstract}
Although learning in high dimensions is commonly believed to suffer from the curse of dimensionality, modern machine learning methods often exhibit an astonishing power to tackle a wide range of challenging real-world learning problems without using abundant amounts of data. How exactly these methods break this curse remains a fundamental open question in the theory of deep learning. While previous efforts have investigated this question by studying the data ($\mathcal D$), model ($\mathcal M$), and inference algorithm ($\mathcal I$) as independent modules, in this paper, we analyze the triplet \dmi as an integrated system and identify important synergies that help mitigate the curse of dimensionality. 
We first study the basic symmetries associated with various learning algorithms ($\M$, $\I$),  
focusing on four prototypical architectures in deep learning: fully-connected networks ($\fcnnn$), locally-connected networks ($\lcnnn$), and convolutional networks with and without pooling ($\gapnn$/$\vecnn$). We find that learning is most efficient when these symmetries are compatible with those of the data distribution and that performance significantly deteriorates when any member of the \dmi triplet is inconsistent or suboptimal.

\end{abstract}

\section{Introduction}

Statistical problems with high-dimensional data are frequently plagued by the \emph{curse of dimensionality}, in which the number of samples required to solve the problem grows rapidly with the dimensionality of the input. The classical theory explains this phenomenon as the consequence of basic geometric and algebraic properties of high-dimensional spaces; for example, the number of $\epsilon$-cubes inside a unit cube in $\R^{d}$ grows exponentially like $\epsilon^{-d}$, and the number of degree $r$ polynomials in $\R^d$ grows like a power-law $d^r$. Since for real-world problems $d$ is typically in the hundreds or thousands, classical wisdom suggests that learning is likely to be infeasible. However, building off of the groundbreaking work AlexNet~\citep{krizhevsky2012imagenet}, practitioners in deep learning have tackled a wide range of difficult real-world learning problems~\citep{vaswani2017attention, devlin2018bert, silver2016mastering, senior2020improved, kaplan2020scaling} in high dimensions, once believed by many to be out-of-scope of current techniques. The astonishing success of modern machine learning methods clearly contradicts the curse of dimensionality and therefore poses the fundamental question: mathematically, how do modern machine learning methods break the curse of dimensionality?   

To answer this question, we examine the most fundamental ingredients of machine learning methods. They are the data ($\D$), the model ($\M$), and the inference algorithm ($\I$).

Data ($\D$) is of course central in machine learning. In the classical learning theory setting, the learning objective usually has a power-law decay $m^{-\alpha}$ as the function of the number of training samples $m$. 
The theoretical bound on $\alpha$ is usually small (e.g.,  $\alpha=d^{-1}$ \citep{von2004distance, bach2017breaking, bahri2021explaining}) and is of limited practical utility for high-dimensional data. On the other hand, empirical measurements of $\alpha$ in state-of-the-art deep learning models typically reveal values of $\alpha$ that are significantly larger (e.g.,  $\alpha=0.43$ for ResNets trained on ImageNet in Fig.\ref{fig:imagenet scaling}) even though $d$ is quite large (e.g.,  $d\sim10^5$ for ImageNet). This example suggests that the learning curve must have important functional dependence on $\M$ and $\I$. Indeed, as we will observe later, many of the best performing methods exhibit learning curves for which $\alpha=\alpha(m)$ actually \emph{increases} as $m$ becomes larger, i.e.,  data makes the usage of data more efficient. We call this phenomenon DIDE, for {\bf d}ata {\bf i}mproves {\bf d}ata {\bf e}fficiency; see Fig.\ref{fig:imagenet scaling}.

Designing machine learning models ($\M$) that maximize data efficiency is critical to the success of solving real-world tasks. Indeed, breakthroughs in machine learning are often driven by novel architectures such as LeNet \citep{lecun1998gradient}, AlexNet\cite{krizhevsky2012imagenet}, Transformer \citep{vaswani2017attention}, etc. While some of the inductive biases of these methods are clear (e.g.,  translation symmetries of CNNs), others tend to build off of prior empirical success and are less well-understood (e.g.,  the implicit bias of SGD). To build our understanding of these biases and how they affect learning, we conduct a theoretical analysis of them in the infinite-width setting \citep{neal, poole2016exponential, daniely2016toward, jacot2018neural, lee2019wide}, which preserves the most salient aspects of the architecture while enabling tractable calculations.

The inference procedure ($\I$) is what enables \emph{learning} in machine learning methods. It is widely believed that modern inference methods, specifically gradient descent and variants, ``implicitly" bias the solutions of the networks towards those that generalize well and away from those that generalize poorly~\citep{neyshabur2017implicit, gunasekar2018implicit, soudry2018implicit}. The effects of the inference algorithm are intimately tied to the specifics of the model (e.g.,  weight-sharing) and the data (e.g.,  data augmentation), and might not be fully understood with a fixed-data, fixed-model analysis. Indeed, good performance may derive from interactions between $(\M, \I)$, or $(\D, \I)$, or even $(\D, \M, \I)$. In Sec.~\ref{sec: dide}, we demonstrate the DIDE effect for a particular choice of $(\D,\M, \I)$ and show that this effect disappears if any one of $\D$, $\M$, or $\I$ is altered. 

The above discussion highlights the insufficiency of treating $\D$, $\M$, and $\I$ as separate non-interacting modules. They must be considered as an integrated system. Throughout this paper, we will refer to the triplet \dmi~as a (machine) learning system and the tuple $(\M, \I)$ as the learning algorithm of the system that operates on $\D$.

We focus our study on four prototypical deep learning architectures whose similarities and differences provide a rich test bed for analysis: fully-connected networks ($\fcnnn$), locally-connected networks ($\lcnnn$), convolutional networks with readout vectorization ($\vecnn$), and convolutional networks with readout global average pooling ($\gapnn$). We consider both finite-width ($\fcnn$, $\lcnn$, $\vecn$, $\gapn$) and infinite-width variants of these architectures ($\fcn$, $\lcn$, $\vecinf$, $\gap$). We examine the basic symmetries of the \dmi~triplets associated to these architectures and find that better architectures break spurious symmetries. We also identify a symmetry breaking effect due to finite-width and carefully examine the impact of this effect to the performance of the system.

\section{Related Work}
The study of infinite networks dates back to the seminal work of~\citet{neal} who showed the convergence of single hidden-layer networks to Gaussian Processes (GPs). Recently, there has been renewed interest in studying random, infinite networks starting with concurrent work on ``conjugate kernels''~\citep{daniely2016toward, daniely2017sgd} and ``mean-field theory''~\citep{poole2016exponential,schoenholz2016}, 
taking a statistical learning and statistical physics view of points, respectively. Since then, this analysis has been extended to include a wide range of architectures ~\citep{lee2018deep, matthews2018, xiao18a, novak2018bayesian, yang2019scaling, hron2020infinite}. The inducing kernel is often referred to as the Neural Network Gaussian Process (NNGP) kernel. The neural tangent kernel (NTK), first introduced in ~\citet{Jacot2018ntk}, along with follow-up work~\citep{lee2019wide, chizat2019lazy} showed that the distribution of functions induced by gradient descent for infinite-width networks is a Gaussian Process with NTK as the kernel. Since then, NNGP and NTK have become extremely useful and popular tools to study various properties of neural networks ~\citep{arora2019finegrained, adlam2020neural, bordelon2021spectrum, mei2021learning, bietti2021approximation, favero2021locality, xiao2021eigenspace} and many others. 

The implicit bias of gradient descent has been the focus of a number of recent works~\citep{soudry2018implicit, Lyu2020Gradient, ji2019implicit,ji2019gradient, chizat20a}, leading to a variety of noteworthy conclusions, including the convergence of GD to the maximal margin solution for logistic-type losses during late-time training~\citep{soudry2018implicit}. \citet{nakkiran2019sgd, hu2020surprising, rahaman2019spectral, xu2018understanding,xu2019frequency, su2019learning, yang2019fine} study the early-time dynamics and spectral biases of neural networks, leading to the conclusion that simpler functions are usually learned before more complex functions.

Understanding and exploiting the structural information in natural data are central aspects of designing machine learning systems. \citet{li2018measuring, goldt2020modeling, pope2021intrinsic} study the low-dimensional structure of natural data while \citet{bruna2013invariant, petrini2021relative} investigate the role of deformation stability of natural data. Designing networks that maximally respect the symmetries of natural data (e.g.,  translational invariance/equivalence of images~\citep{cohen2016group, zaheer2017deep}) is widely considered a principled approach in practice. Several works also demonstrate the possibility of learning such symmetries from scratch using natural or synthetic data~\citep{neyshabur2020towards, ingrosso2022data}. Nevertheless, recent breakthroughs in applying attention-based models \cite{dosovitskiy2020image, tolstikhin2021mlp, he2021masked} to computer vision have fundamentally challenged the significance that symmetries play in model design. Attention-based models have weaker inductive biases (lacking even translation equivariance) than those of convolutional networks, yet their performance is comparable in the large data regime \citep{zhai2022scaling}. Our analysis of DIDE (Fig.~\ref{fig:imagenet scaling}) sheds some light on why such models are able to reach good performance with the help of more data.

\section{Preliminaries and Notation}
We focus our presentation on the supervised learning setting and more concretely, on image recognition. Let $\D \subseteq \mathbb (\mathbb R^{d})^3 \times \mathbb R^{k}\equiv \mathbb R^{3d} \times \mathbb R^{k}$ denote the data set (training and test) and $\X=\left\{x: (x,y)\in \D\right\}$ and $\Y=\left\{y: (x,y)\in \D\right\}$ denote the input space (images) and label space, respectively. Here $d$ is the spatial dimension (e.g.,  $d=32\times 32$ for CIFAR-10) of the images and $3$ is the total number of channels (i.e.,  RGB). We assume $(x,y)\in\mathcal D$ is obtained from some data generating process with unknown distribution $\mu_{\mathcal D}$ and the learning task is to recover $\mu_{\mathcal D}$.

\subsection{Neural Networks}
\label{sec:notation}
We use $\fcnn$ to denote an $L$-hidden layer fully-connected network with identical hidden widths $n_l=n\in\mathbb N$ for $l = 1, ..., L$ and with readout width $n_{L+1} = k$ (the number of logits). For each $x\in\mathbb R^{3d}$, we use $h^l(x), x^l(x)\in\mathbb R^{n_l}$ to represent the pre- and post-activation functions at layer $l$ with input $x$. The recurrence relation $\fcnn$ is given by 
    \begin{align}
    \label{eq:recurrence}
    \begin{cases}
        h^{l+1}&=x^l W^{l+1} 
        \\
        x^{l+1}&=\phi\left(h^{l+1}\right) 
        \end{cases}
        \,\, \textrm{and} 
        \,\,
      W^{l+1}_{i, j}& = \frac {1} {\sqrt{n_{l+1}}}  \omega_{ij}^{l+1} 
    \end{align}
        where
        $\phi$ is a point-wise activation function, $W^{l+1}\in \mathbb R^{n_l\times n_{l+1}}$ are the weights and $\omega_{ij}^l$ are the trainable parameters, drawn i.i.d. from a standard Gaussian $\mathcal N(0, 1)$ at initialization. For simplicity of the presentation, the bias terms and the hyperparameters (the variances of the weights) are omitted; including them as hyperparameters will not significantly change any of the main conclusions.
        
        For convolutional networks, the inputs are treated as tensors in $(\mathbb R^d)^3$. The recurrence relation of convolutional networks can be written as 
        \begin{align}
        \begin{cases}
                x^{l+1}_{\alpha, j} &= \phi(h^{l+1}_{\alpha,j}) 
\\
 h^{l+1}_{\alpha, j} &= \frac 1 {\sqrt{(2k +1)n^{l}}}\sum_{i=1}^{n^{l}} \sum_{\beta = -k}^{k}  x^l_{\alpha + \beta, i} \omega^{l+1}_{ij, \beta}\,.
            \label{eq weighted sum}
        \end{cases} 
    \end{align}
        Here $\alpha\in [d]$ denote the spatial location, $i/j\in [n]$ denotes the fan-in/fan-out channel indices. For notational convenience, we assume circular padding and stride equal to 1 for all layers. The features of the penultimate layer are 2D tensors and there are two commonly used approaches to map them to the logit layer: (a) $\vecn$, which vectorizes the 2D tensor to a 1D vector, yielding a translation-equivariant inductive bias, or (b) $\gapn$, which applies a global average pooling layer to each channel, yielding a translation-invariant inductive bias. The readout layers for these models can be written as,
        \begin{alignat}{2}
            x^{L+1}_j &= \frac{1}{\sqrt{dn}} \sum_{i\in [n]}\sum_{\alpha\in [d]} x^{L}_{\alpha, i} w_{ \alpha, ij}^{L+1}\quad&(\vecn)\,, 
             \\
            x^{L+1}_j &= \frac{1}{\sqrt{n}} \sum_{i\in [n]}\left(\frac 1 {d}\sum_{\alpha\in [d]} x^{L}_{\alpha, i}\right) w_{ ij}^{L+1}\quad&(\gapn)\,.
        \end{alignat}
         The key difference between the two architectures is that, in $\vecn$, each pixel in the penultimate layer has its own readout variable, whereas in $\gapn$ the pixels within the same channel share the same readout variable. It is clear that the function space of $\vecn$ contains that of $\gapn$. 
        
        Locally Connected Networks ($\lcnn$) \citep{fukushima1975cognitron, lecun1989generalization} are convolutional networks {\it without} weight sharing between spatial locations. They share the connectivity pattern and topology of a standard convolutional network, but the weights are not shared across spatial patches. Mathematically, the network is defined as in \eqref{eq weighted sum}, but with all the {\it shared} parameters $\omega_{ij, \beta}^{l}$ replaced by {\it unshared} $\omega_{ij, \alpha, \beta}^{l}\sim \mathcal N(0, 1)$. In this work, we assume that the $\lcnn$ are always associated with a vectorization readout layer and it is therefore clear that the function space of $\lcnn$ is a superset  of $\vecn$. Interestingly, $\lcnn$ is also a subset of $\vecnn_{dn}$ when the network is $d$ times wider.
    
    \begin{figure*}[t]
        \begin{center}
        \includegraphics[width=.9\textwidth]{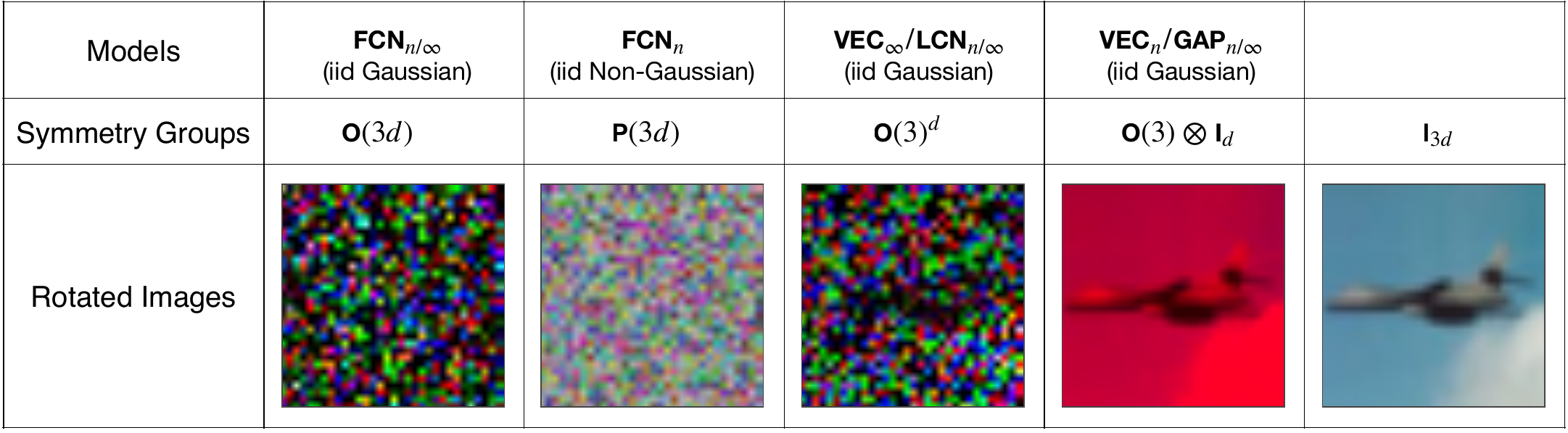}
        \caption{{\bf Models and associated symmetry groups.}
        Top: architectures (initialization scheme). 
        Middle: implied symmetry groups. 
        Bottom: rotated images, where the rotation is randomly drawn from the corresponding symmetry group.
        The largest symmetry group for which a random transformation does not obfuscate the underlying image is evidently $\Othree$.
        }
        \label{fig:CIFAR-10-rotation}
        \end{center}
    \end{figure*}
        \begin{theorem}[Sec.~\ref{sec:proof of inclusion thm}]\label{thm:change of prior}
        Let ${\mathsf { VEC} }_n /{\mathsf {LCN} }_{n}/\gapn/\fcnn$ denote the set of functions that can be represented by a $L$-hidden layer ${\mathsf { VEC} }_n /{\mathsf {LCN} }_{n}/\gapn/\fcnn$ network with hidden-layer width $n$. Then
        \begin{align}
        \gapn \subseteq {\mathsf { VEC} }_n \subseteq  {\mathsf {LCN} }_{n}
        \subseteq {\mathsf {VEC} }_{dn} \subseteq {\mathsf {FCN} }_{d^2n}\,.
        \end{align}
        \end{theorem}
        As remarked above, the random initialization of parameters endows $\gapn$ with a translation-invariant prior, which may be well-suited to many image-classification tasks. This observation, combined with the result from~\cref{thm:change of prior} that the function space of $\gapn$ is the smallest, suggests that networks with this architecture may enjoy favorable generalization properties. Indeed, prior work~\cite{novak2018bayesian,lee2020finite,neyshabur2020towards} has found that $\gapn$ can significantly outperform $\vecn$, $\lcnn$, and $\fcnn$, a conclusion we also find in Sec.~\ref{sec: empirical}.
        
        We emphasize that the above observation relies on a notion of the prior induced by initialization and says little about the effect of optimization. It is possible that gradient descent could update the readout weights of a network from the $\vecn$ class toward a configuration that approximately implements average pooling, thereby pushing the model closer to a member of the $\gapn$ class. Alternatively, if the weights remain close to their initial random values, the function might more closely resemble a member of the $\lcnn$ class. This perspective gives some intuition for how the inference algorithm $\I$ can interact with the model $\M$ to produce predictive functions with significantly different generalization properties. We return to this theme in Sec.~\ref{sec: dide}.

    \subsection{Infinite Network: Gaussian Processes and the Neural Tangent Kernels}
    In order to better disambiguate the effects of the model $\M$ from the inference algorithm $\I$, it is useful to examine our model families in the limit of infinite width. This limit facilitates simpler theoretical analysis while simultaneously preserving most of the salient ingredients of the finite-width models. Below, we briefly review several useful aspects of infinitely-wide networks and refer interested readers to the references for a more comprehensive introduction.  
    
    \paragraph{Neural Networks as Gaussian Processes (NNGP).}
    As the width $n\to\infty$, at initialization the output $f_0(\X)$ forms a Gaussian Process $f_0(\X)\sim \mathcal {GP}(0, \K(\X, \X))$, known as the NNGP \citep{neal, lee2018deep, matthews2018}. Here $\K$ is the GP kernel and can be computed in closed-form for a variety of architectures \citep{novak2019neural}. By treating this infinite width network as a Bayesian model (aka Bayesian Neural Networks) and marginalizing over the training set $(\X_T, \Y_T)$, the posterior is also a GP whose mean is given by  $\K(\X_*, \X_T)\K^{-1}(\X_T, \X_T)\Y_T$. 

    \paragraph{Neural Tangent Kernels (NTK).}
    Recent advances in the global convergence theory of large width networks~\citep{Jacot2018ntk, du2018global, allen2018convergence, zou2020gradient, lee2019wide} have shown that under certain assumptions, the gradient descent dynamics of a network converge to that of kernel gradient descent as the widths approach infinity, where the kernel is the NTK \citep{Jacot2018ntk}, denoted by $\Theta$.  
    As such, in the infinite width limit, when the loss is the mean squared error (MSE), the mean prediction (marginarized over random initialization) can be solved analytically. In particular, when the training time $t=\infty$, the prediction is given by 
        \begin{align}
      &f(\X_*) =\infntk\left(\X_*, \X_T\right)\infntk^{-1}(\X_T, \X_T)\Y_T \,,
      \label{eq:lin-exact-dynamics-mean}
      \end{align}
For convenience, we use $\fcn(x)$, $\lcn(x)$, $\vecinf(x)$ and $\gap(x)$ to denote the infinite width solutions (either via NNGP inference or NTK regression) for the corresponding architectures.

    \section{Symmetries of Machine Learning Systems}

    In this section, we study the symmetry properties of various machine learning systems \dmi, focusing on how invariances implied by the learning algorithm $(\M, \I)$ interact with the data distribution $\mu_{\mathcal D}$.

    To provide concrete context, let us consider solving image classification using kernel regression with an inner-product kernel $K$ (i.e.,  $K(x, \bar x) = k(\langle x,  \bar x\rangle)$ for some function $k$). Since $K$ is invariant to rotations of its inputs (i.e.,  for any rotation operator $\tau$, $K(\tau x, \tau \bar x)= K(x, \bar x)$), if we rotate all (both train and test) images by any fixed rotation $\tau$, the kernel is unchanged and so are the predictions. Because natural images surely exhibit spatial structure that is destroyed by rotations (see Fig.~\ref{fig:CIFAR-10-rotation}), we might expect such kernels to perform poorly on image classification tasks, and indeed we find this to be the case (see Sec.~\ref{sec: empirical}). In this sense, the symmetry properties implied by the learning algorithm are incompatible with the data distribution $\mu_D$, and we regard them as \emph{spurious}. Below we describe some notation and results that will help us analyze this type of interaction more systematically.

    For a deterministic (stochastic) learning algorithm $\mathcal A = (\M, \I)$, we use $\mathcal A(\D_T)$ to denote the learned function (distribution of the learned functions) using training set $\D_T$. We use $\mathcal A^\tau(\D_T)$ to denote the learned function(s) using $\tau(\D_T)$, which make predictions on the transformed test point $\tau(\X_*)$. In other words, 
    all inputs, including training and test inputs, are pre-processed by a common transformation $\tau$ before feeding to the learning algorithm $\mathcal A$.  
    For convenience, for random variables $A$ and $B$, we use $A=^d B$ to indicate that they have the same distribution.
    \begin{definition}
    Let $\mathcal G$ be a group of transformations $\mathbb R^{3d}\to \mathbb R^{3d}$. A deterministic (stochastic) learning algorithm $\mathcal A =(\M, \I)$ is $g$-invariant if $\mathcal A = \mathcal A^g$ ($\mathcal A =^d \mathcal A^g$). In this case, we say the system $(\D, \M, \I)$ is $g$-invariant and use the notation $(\D, \M, \I) = (g\D, \M, \I)$. If this holds for all $g\in\mathcal G$, then we say the algorithm and the system are $\mathcal G$-invariant.  
    \end{definition}

    \paragraph{Comparing with Functional Invariance.} 
    The flavor of invariance studied this paper is {\it algorithmic} invariance, as it concerns a \emph{system} or a \emph{learning algorithm}, and is qualitatively different from the {\it functional} invariance studied elsewhere~\citep{cohen2016group}. 
    Recall that a function $f$ is (functionally) invariant to a group $\gG$ if $f(\tau x) = f(x)$ for all $\tau\in \gG$. Natural images are often considered to be translationally invariant, which is a key motivation for the usage of {\it convolutional} networks in computer vision. These symmetries are {\it hard-coded} into the architectures\footnote{In the idealized setting when circular padding is applied and the readout layer is a global average pooling.} as a kind of {\it inductive bias} and the (post-pooling) representations are invariant to this group by design. As a consequence, the corresponding hypothesis class, defined by these networks, is more restricted, which could lead to better generalization performance if the symmetry were exact~\citep{shalev2014understanding}. 
    Algorithmic invariance is weaker than functional invariance. 
    For example, in the discussion above, the learning system defined by kernel regression with an inner product kernel is algorithmically invariant to the rotational group $\gG$ since $K(\tau x, \tau \bar x)= K(x, \bar x)$ for all $\tau\in\gG$, but the learned function itself is not \emph{functionally} invariant because $K(\tau x, \bar x)\neq K(x, \bar x)$.\footnote{This stronger condition could be achieved by defining an invariant kernel $K^{\text{inv}}(x, \bar x)= \int_{\tau\in\gG}K(x, \tau\bar x)d\tau$, or approximated by augmenting the training set by group actions as in~\cite{chen2020group}, but we do not pursue this here} 
    
    To present the implied symmetry properties of the main architectures under study, we need to introduce some notation. Let $\Odthree$ denote the orthogonal group on the flattened input space $\mathbb R^{3d}$. The subgroup $\Othreed \leq \Odthree$ operates on the un-flattened input $(\mathbb R^3)^d$, whose elements rotate each pixel $x_\alpha\in\mathbb R^3$ by an independent element $\tau_\alpha\in\text{O}(3)$. The smaller subgroup $\Othree \leq\Othreed$ applies the {\it shared} rotation, i.e.,  $\tau_\alpha=\tau$ to all $x_\alpha$ for $\alpha\in [d]$. Similarly, we use $\text{P}(3d)$ to denote the permutation group on $\mathbb R^{3d}$ and $\text{P}(3)^{d}$ and $\text{P}(3)\otimes {\bf I}_d$ are defined similarly. Note that rotating $\X$ by $\tau$ is equivalent to transforming the underlining coordinate systems of the input by the adjoint $\tau^*=\tau^{-1}$. Fig~\ref{fig:CIFAR-10-rotation} displays an image from CIFAR-10 under five families of rotations. We use 
    $\fcnnn_{n}, \vecn$, etc. to denote the output function (distribution) of a finite-width network obtained by SGD, in which the random initialization is the only source of randomness.

    \begin{theorem}[Sec.\ref{sec:proof of symmetries}]
    \label{thm: rotation}
    If the initial parameters of the networks defined in Sec.~\ref{sec:notation} are iid samples from $\mathcal N(0, 1)$, then the predictions from finite-width networks trained by SGD or infinite-width networks trained by kernel regression enjoy the following symmetries: 
\end{theorem}

\begin{itemize}
  \item
        $\fcnnn_{n/\infty}$ are $\Odthree$-invariant
    \item
        $\lcnnn_{n/\infty}$ and $\vecinf$ are $\Othreed$-invariant
    \item
        $\vecn$ and $\gapnn_{n/\infty}$ are $\Othree$-invariant.
\end{itemize}

    The $\Odthree$-invariance of $\fcnnn_{n/\infty}$ follows from the rotational invariance of the Gaussian measure, and has been observed in many prior works, including~\citep{wadia2020whitening, li2020convolutional}. Rotating the input by $\tau \in\Odthree$ is equivalent to rotating the weight matrix $\omega$ of the first layer by $\tau^*$, and since $\tau^*\omega =^{d}\omega$ for $\omega\sim\mathcal N(0, 1)^{3d}$, the distribution of the output functions at step 0 (aka initialization) is invariant. This observation implies that the first gradient is also $\Odthree$-invariant, which further implies the $\Odthree$-invariance of the output function after the first gradient update. By induction, this invariance property holds throughout the course of gradient descent training, even with $L^2$-regularization as the $L^2$-norm is rotationally invariant. Such invariant property also holds for (finite-width) Bayesian posterior inference thanks to the Bayes theorem: $ P(\tau \gX^*| \tau \gX_T)  =  P( \gX^*| \gX_T) $ because $ P(\tau \gX^*,  \tau \gX_T)/P(\tau \gX_T) = P( \gX^*, \gX_T)/P(\gX_T)$.

    For the same reason, $\lcnn$ is $\Othreed$-invariant because each patch of the image uses independent Gaussian random variables. In addition, weight-sharing in $\vecn$ and $\gapn$ breaks the $\Othreed$ symmetry, reducing it to $\Othree$. 

    For infinite networks, $\lcn= \vecinf$ \citep{xiao2018dynamical,novak2018bayesian, garriga2018deep}. The kernels
    of $\vecinf$ and $\gap$ are of the forms 
    \begin{align}
        \infntk_{\vecnn}(x, x') &= k(\{\langle x_\alpha, x'_\alpha\rangle \}_{\alpha\in [d]}) 
        \\
        \infntk_{\gapnn}(x, x') &= k(\{\langle x_\alpha, x'_{\alpha'}\rangle\}_{\alpha, \alpha'\in [d]}).
    \end{align}
    The former depends only on the inner product between pixels in the {\it same} spatial location,  
    breaking the $\Odthree$ symmetry of $\fcn$ and reducing it to $\Othreed$. In addition, the latter depends 
    also on the inner products of pixels across different spatial locations due to pooling, which breaks the $\Othreed$ symmetry and reduces it to $\Othree$. Noting that 
    ${\dim}(\Odthree) = 3d(3d-1)/2 > {\dim}(\Othreed) = 3d > \text{dim}(\Othree)=3$, we see that $\lcnn/\vecinf$ dramatically reduces the dimensionality of the symmetry group. As we will observe in Secs.~\ref{sec: empirical} and~\ref{sec: dide}, if a symmetry group is inconsistent with the data distribution, the performance of the associated learning algorithm tends to diminish in proportion to the dimension of the spurious symmetry group; see Fig.~\ref{fig:rotation_vs_accuracy}.
    
    The results of the paper are presented in the most {\it vanilla} setting. Our methods can easily extend to more complicated architectures like ResNet\citep{he2016deep}, MLP-Mixer\citep{tolstikhin2021mlp}, etc. The symmetry groups of such systems need to be computed in a case-by-case manner by identifying the invariant group of the random initialization and training procedures. 
    For example, the orthogonal group type of symmetries needed to be replaced by the permutation-type of symmetries if non-Gaussian i.i.d. initialization or/and $L^p$ ($p\neq 2$) regularization. However, we empirically observe that swapping the Gaussian initialization by the uniform initialization in the first layer does not essentially change the performance of the network; see Sec.~\ref{sec:plots dump}. This observation indicates that the permutation group may exhibit a similar degree of spuriousness as the rotation group; however, more rigorous and thorough experiments are needed to confidently confirm this claim, which is left for future work.
    Moreover,  the invariance property studied here is mainly coming from the first layer and it is possible that later layers could contribute new invariances to the system. For example, owing to the non-overlapping between patches in ViT \citep{dosovitskiy2020image}, there could be permutation symmetries between the patches in the subsequent self-attention layer (assuming no positional encoding). 
    Finally, for the sake of simplicity, we use NTK-parameterization \citep{Jacot2018ntk} but our results apply to other network parameterizations, including standard- \citep{sohl2020infinite}, meanfield- \citep{mei2018mean}, $\mu$-parameterizations \citep{yang2020feature}. 
    In particular, both finite- and infinite-width FCNs still suffer from the most spurious symmetries $\Odthree$ for all such parameterization schemes, which may explain the poor performance of FCN in the ``feature learning" regime (e.g.,  $61.5\%$ accuracy on CIFAR-10, Table 1. in \citet{yang2022efficient}.)

    \section{Empirical Analysis}\label{sec: empirical}

    \begin{figure*}[t]
    \centering
    \includegraphics[width=0.32\textwidth]{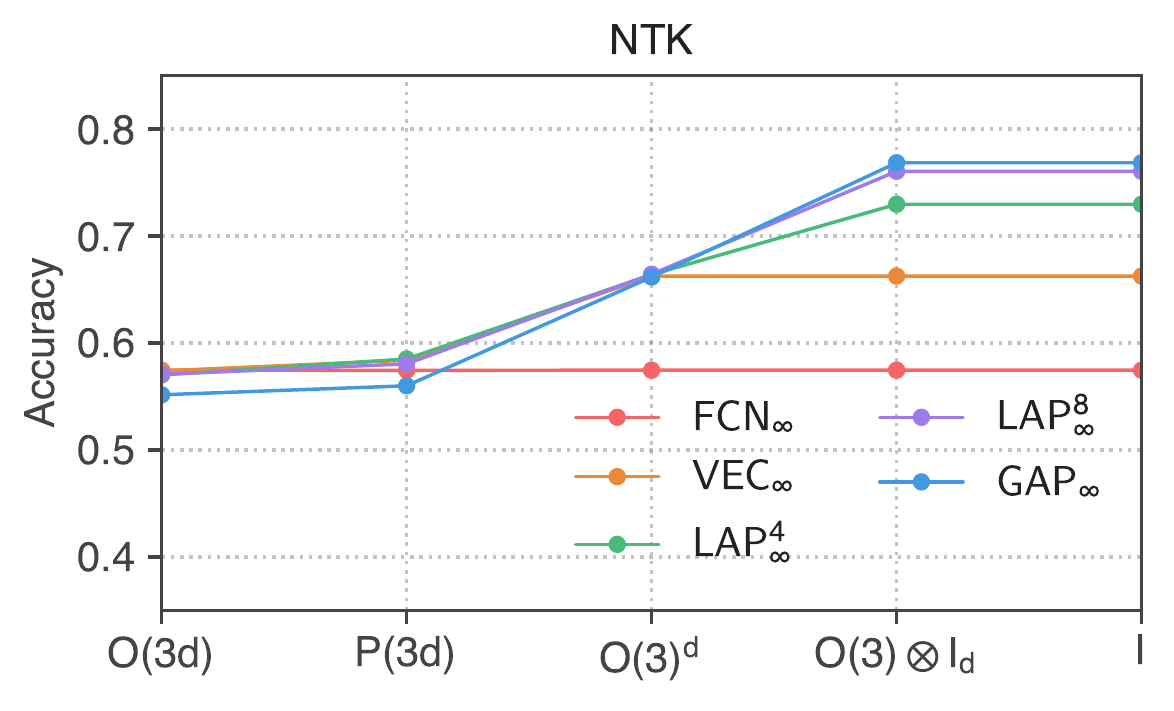}
    \includegraphics[width=0.32\textwidth]{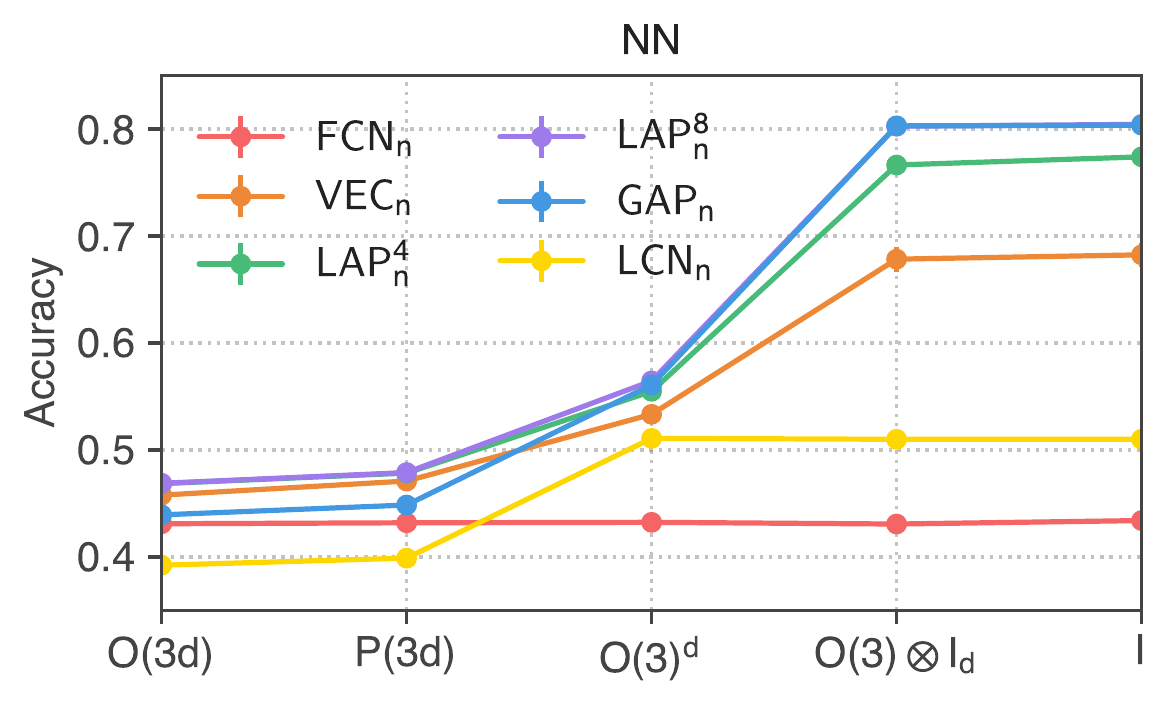}
    \includegraphics[width=0.32\textwidth]{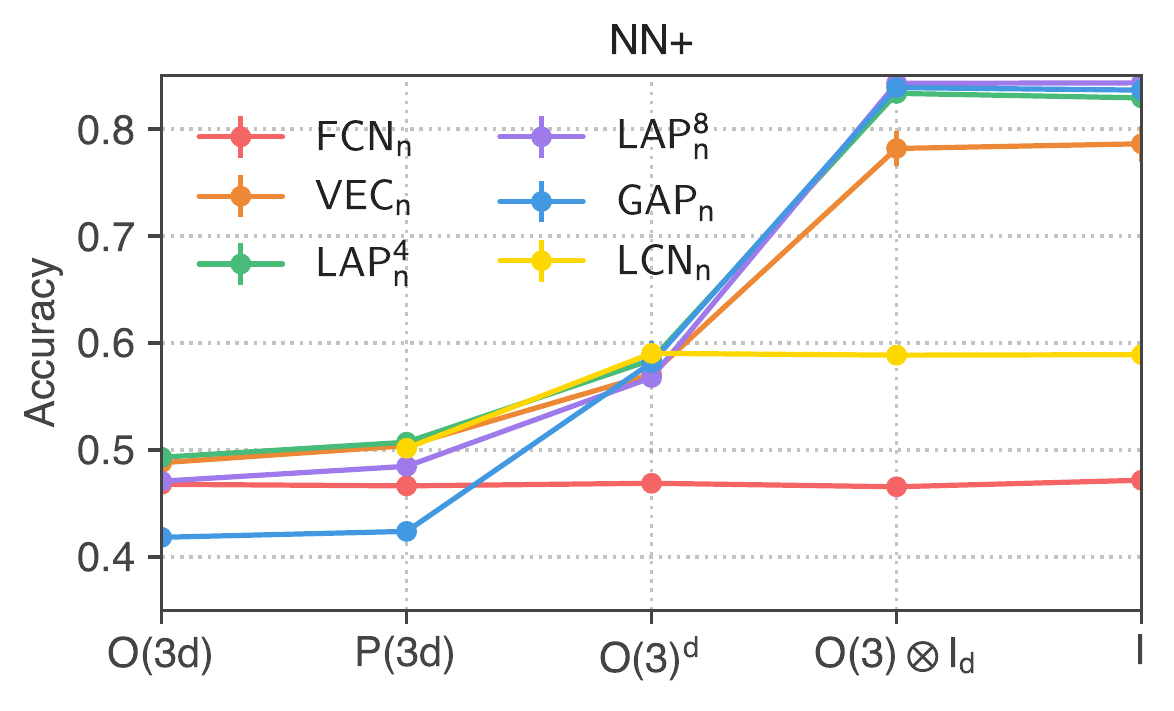}
    \caption{
    {\bf Performance of various architectures under different data transformations}.
    Left: network trained using NTK. Middle: finite-width network trained with a small learning rate and no L2-regularization. Right: larger learning rate and L2-regularization.  
    We transform all images in CIFAR-10 by a random element from one of the five groups ($x$-axis) and plot the accuracy ($y$-axis) for 6 architectures. Performance degrades in the presence of spurious symmetry, and the decrease is similar regardless of whether that symmetry arises from the data or from the model (see Thm.~\ref{thm: rotation}).
    }
    \label{fig:rotation_vs_accuracy}
\end{figure*}  

    This section focuses on empirical analysis. Details regarding the experimental setup of this and next section can be found in Sec.~\ref{sec:experimental details}. The goal is to (1) verify Theorem \ref{thm: rotation}, (2) study the consistency between $\D$ and $\mathcal A=(\M, \I)$ via the lens of symmetries, and (3) study the effect of symmetry breaking by comparing $\vecn$ to $\vecinf$. In light of the visualization of the images in Fig.~\ref{fig:CIFAR-10-rotation}, as well as the numerical performance of the various methods mentioned below, $\Othree$ is the largest symmetry group that is compatible with $\mu_D$. In what follows, we regard all larger symmetries ($\Othreed$, $\text{P}(3d)$ and $\Odthree$) as spurious. 
    \subsection{Experimental Setup.}

    We conduct experiments on $\D= \text{CIFAR-10}$ \citep{krizhevsky2009learning}, which is a standard image dataset that consists of $50,000/10,000$ training/test images. We vary each member of \dmi as follows. {\bf Five Datasets.} We create 5 families of new datasets $\tau D$ by rotating all input images in $\D$ by $\tau$, a fixed random element of one of the five groups: identity operator ${\bf I}_{3d}$, pixel-wise shared-rotations, $\Othree$, pixel-wise (unshared-)rotations $\Othreed$, permutations $\text{P}(3d)$ and global rotations $\Odthree$.  {\bf Six Models.} In addition to $\fcnnn$, $\lcnnn$, $\vecnn$ and $\gapnn$, we add $\mathsf{LAP}^{4/8}$ which are the same as $\gapnn$ except the readout layer is replaced by the {\bf L}ocal {\bf A}verage {\bf P}ooling with window size $4\times 4/8\times 8$. All networks have 8 hidden layers. {\bf Three Inference Algorithms.} (1) NTK regression\footnote{
    When investigating the impact of inference algorithms, it would be preferable to compare finite-width Bayesian inference to SGD. Unfortunately, Bayesian inference is too expensive to perform exactly and approximations may induce unwanted biases. As such, we instead use infinite-width Bayesian inference (i.e.,  NNGP regression~\citep{hron2020exact}), whose performance is usually very similar to NTK regression.} (aka infinite-width networks), (2) NN, our baseline for finite-width networks which is trained with momentum using a {\it small learning rate} without $L^2$ regularization; (3) NN+:= NN$\LR \REG$, i.e.,  using a larger learning rate ($\mathsf{+LR}$) and adding $L^2$ regularization ($\mathsf{+L2}$). We plot the test accuracy for each \dmi (a total of $90=3\times 6\times 5$) in Fig.~\ref{fig:rotation_vs_accuracy}. The accuracy for finite-width networks is averaged over five runs\footnote{A few \dmi have only 1 or 2 successful runs.} for each \dmi, and in each run the rotation $\tau$ is resampled. Note that the total variance across runs is small, indicating that the particular choice of $\tau$ has a negligible effect on the results.

    \subsection{Verifying Theorem \ref{thm: rotation}} 
    As expected from Theorem \ref{thm: rotation}, Fig.~\ref{fig:rotation_vs_accuracy} shows that across $\ntk$/NN/NN+, the performance of $\fcnnn_{n/\infty}$ is invariant to all symmetry transformations, the performance of $\lcnnn_{n}$ and $\vecinf=\lcnnn_\infty$ are invariant to $\Othreed$ (and its subgroups), and the performance of $\vecn$, $\mathsf{LAP}^{4/8}$ and $\gapnn_{n/\infty}$ are invariant to $\Othree$.
    
    \subsection{Effect of spurious symmetries}
    In order to analyze the consistency of the data $\D$ and algorithm $\mathcal A$, we examine performance in the presence of various spurious symmetries, which we introduce through the five rotated datasets and through the six different model families. We focus on two main findings: (1) the strength of the spurious symmetry controls performance, regardless of how it is introduced; and (2) SGD confers its main benefits in the absence of spurious symmetries.
    
    \paragraph{Performance dictated by spurious symmetries.} For each fixed $\I$, we choose the highest-performing triplet $(\D, \gapnn, \I)$ as a baseline that exhibits the strongest consistency between $\mathcal A$ and $\mu_D$. We then progressively break this consistency by injecting spurious symmetries in two ways: (1) fixing $(\gapnn, \I)$ and changing the dataset to $\tau\D$; and (2) fixing $(\D, \I)$ and changing $\gapnn$ to $\M_\tau$, where $\M_\tau$ represents an architecture that is $\tau$-invariant (c.f.\cref{thm: rotation}).
    
    From Fig.~\ref{fig:rotation_vs_accuracy}, we see that for each fixed $\I\in \{\ntk, \text{NN}, \text{NN+}\}$, test performance monotonically decreases as the symmetries become more ``spurious''. We also observe that performance is to a good approximation determined by the spurious symmetry itself, independent of the way it was introduced. In particular, across all settings we observe that the performance of $(\tau \D, \M, \I)$ is close to that of $(\D, \M_\tau, \I)$ which is itself nearly the same as $(\tau \D, \M_\tau, \I)$. As a concrete example of this relationship, in Fig.~\ref{fig:rotation_vs_accuracy} for $\I=\ntk$, the performance of $\gap$ under an $\Othreed$ data transformation equals the performance of $\vecinf$ with no data transformation, which is expected since~\cref{thm: rotation} implies $\vecinf$ is $\Othreed$-invariant.

    \paragraph{Spurious symmetries eliminate the benefit of SGD.}
     We examine the benefits of SGD by varying $\I$ from $\ntk$ (no SGD-impact, left panel of Fig.~\ref{fig:rotation_vs_accuracy}) to $\text{NN}$ (weak SGD-impact, middle panel) to $\text{NN+}$ (strong SGD-impact, right panel). We find that the behavior depends strongly on whether or not $(\D, \M)$ has spurious symmetries. In the presence of spurious symmetries, i.e.,  when either $\M$ is invariant to a symmetry group larger than $\Othree$ or when $\D$ is rotated by an element from such a group, there is no benefit from SGD, as we observe that $\ntk$ outperforms both NN and NN+. In the absence of spurious symmetries, i.e.,  when $\M\in \{\gapnn, \vecnn, \mathsf{LAP}^{4/8}\}$ and when the transformation applied to $\D$ is from a group no larger than $\Othree$, we observe a significant boost in performance when changing the inference algorithm from $\I=\ntk$ to $\I=\text{NN+}$, e.g.,  $77\%\to 84\% $ when $\gap\to\gapn$. The performance gain of $\vecnn$ is most significant ($67\% \to 78\%$ when $\vecinf\to\vecn$), which will be discussed in detail in the following section. Overall, our empirical results suggest that SGD does not improve the performance when $(\D, \M)$ has spurious symmetries, at least for CIFAR-10 without data augmentation. We argue that, when studying the benefits of feature learning, it is essential to take into account both the data distribution and the architecture choices.

\subsection{Symmetry breaking at finite width in $\vecn$.} 
Weight-sharing implies that $\vecn$ is $\Othree$-invariant; however, when $n\to\infty$, this symmetry group is enlarged to $\Othreed$ (see~\cref{thm: rotation}). This observation highlights a novel and previously underappreciated aspect of convolutional models with vectorization: symmetry breaking at finite width. As we will see, this broken symmetry can have a significant impact on performance. 

\begin{figure}[t]
\centering
\includegraphics[width=0.23\textwidth]{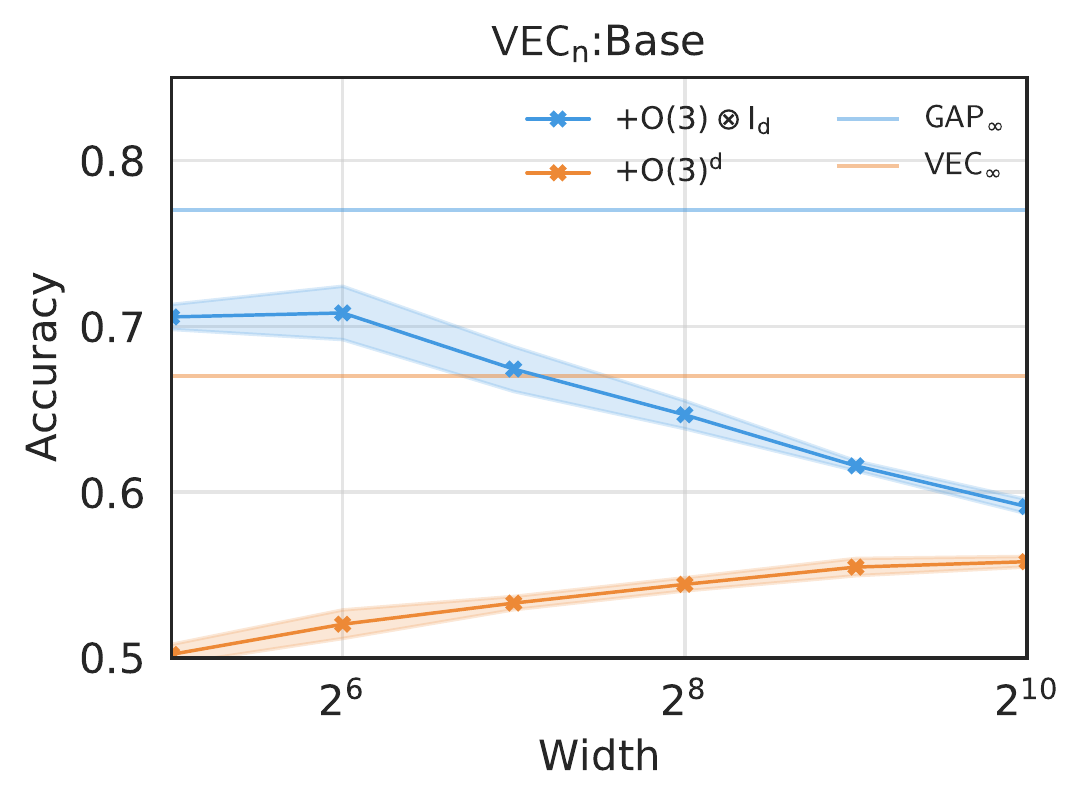}
\includegraphics[width=0.23\textwidth]{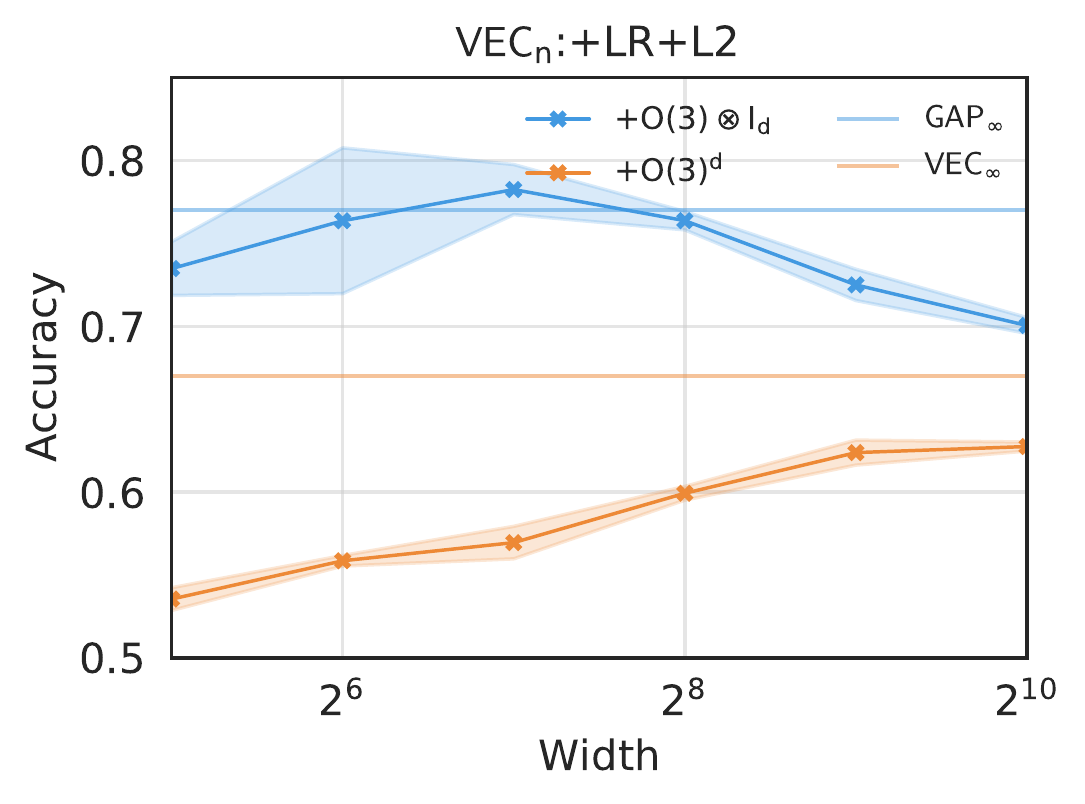}
\caption{{\bf Finite width Effect for $\vecn$.} Left: networks trained with a small learning rate and no L2-regularization. Right: with a larger learning rate and L2-regularization. With {\bf\color{plot_orange}$\Othreed$} imposed on the data, performance of $\vecn$ is far below the performance of $\vecinf$  (67\%, {\bf\color{plot_orange} orange horizontal line}), but improves monotonically as a function of $n$. However, with smaller rotation group {\bf \color{plot_blue} $\Othree$} (has no impact on $\vecn$), performance degrades substantially after the peak ({\bf \color{plot_blue} blue lines}).  This suggests a strong finite-width effect of $\vecn$ due to symmetry breaking.} 
    \label{fig:non-monotonic behavior of CNN-Vec}
\end{figure}

    In Fig.\ref{fig:non-monotonic behavior of CNN-Vec}, we plot the performance on CIFAR-10 of $\vecn$ as a function of width under the interventions: (1) rotating the data by $\Othree$ (which is actually a null operation with no impact) or by $\Othreed$, and (2) training under NN or NN+ (i.e.,  $\LR\REG$).  
    We summarize our findings below.

        {\bf Strong benefits of finite width ({\bf\color{plot_blue} Blue Lines}).} Both with and without the intervention of $\LR\REG$, the performance decreases monotonically and dramatically\footnote{This phenomenon was first observed in \cite{lee2020finite}.} towards or across the $\ntk$ performance ({\bf\color{plot_orange}{Faded Orange Lines}}) when $n$ passes a threshold ($n\approx2^6/2^7$ for NN/NN+).

         {\bf Spurious symmetry eliminates benefits of finite width.} When changing $\D\to \tau \D$ where $\tau\in\Othreed$, a spurious symmetry incompatible with weight-sharing convolution, modern wisdom in machine learning is restored: performance improves with overparameterization, gradually approaching $\ntk$ performance ({\bf\color{plot_orange} Orange Lines}). 
         
         {\bf Large fluctuations at small $n$} ({\bf \color{plot_blue} Blue shaded area} in Fig.\ref{fig:non-monotonic behavior of CNN-Vec}.)
        The intervention $\LR\REG$ not only improves the performance of $\vecn$ but also increases the variance of the performance substantially 
        \footnote{For each configuration, the standard deviation is computed over 5 different random initialization. Some of the runs fail and we only plot the configurations with at least 4 successful runs (best training accuracy is at least 95\%.)} for small width ($n=2^5$, $2^6$). For example, when $n=2^6$ the validation accuracy are 68\%, 78\%, 78\%, 79\%, 79\%. Such variability does not occur in the presence of spurious symmetries ($\Othreed$).

One interpretation of Thm.~\ref{thm:change of prior} and the above observations is that strong algorithmic forces are needed to help $\vecn$ escape from the undesirable function class $\lcnnn$ and move towards the more desirable properties of $\gapnn$. As discussed earlier, one mechanism behind such movement could be learning an approximation of pooling in the readout layer. 

\section{Data Improves Data Efficiency}\label{sec: dide}
In the previous section, we investigated the consistency of various machine learning systems through the lens of symmetries. In this section, we further investigate the interplay of the components of the \dmi triplet by conducting a fine-grained analysis of learning curves on various SoTA vision models. 

Recently, \citet{hoffmann2022training,kaplan2020scaling, bahri2021explaining} and other authors suggest that for real-world problems, the learning objective often has a power-law $\sim m^{-\alpha}$ dependence of training set size $m$, where the exponent $\alpha$ is a constant that usually independent from $m$. A surprising finding is that, for certain triplets \dmi, $\alpha$ can grow as $m$ becomes larger, i.e.,  data improves data efficiency (DIDE); see Fig.~\ref{fig:vec bends learning curve}. In what follows, we first examine the learning curve of $\vecn$ to better understand the huge performance gap between $\vecn$ and $\vecinf$. We then move to SoTA models, in which we observe a ``cusp'' in the learning curve. Finally, we provide possible explanations for the observed phenomena.   

\subsection{DIDE for $\vecn$}\label{sec:dide for vec} 
We vary the training set size of CIFAR-10 from $320$ to 45k (the whole un-augmented training set) and then to 90k (adding left-right flip augmentation) and plot the learning curves in Fig.~\ref{fig:vec bends learning curve} for various $(\D, \mathcal M, \mathcal I)$.
We observe a dramatic speedup of learning in our baseline setting $(\D, \mathcal M, \mathcal I) = (\text{CIFAR-10}, \vecn, \text{SGD})$ ({\bf\color{plot_blue} Blue Lines}). 
Pictorially, the slope of the learning curve is steepened substantially in the log-log plot. We then did an ablation study by changing one member in \dmi at a time: (1) {\bf Inference Algorithm $\I$ ({\bf \color{plot_orange} Orange Lines}):} SGD to NTK\footnote{This requires changing $\vecn$ to $\vecinf$} (2){\bf Model $\M$ ({\bf \color{plot_red} Red Lines}):} $\vecn\to\lcnn$, and (3) {\bf Data $\D$ ({\bf \color{plot_green} Green Lines}):} ${\D} \to \tau\D$, where $\tau\in\Othreed$ randomly selected. In all cases above, this phenomenon disappeared. 

\paragraph{Possible Explanation.} Recall from Theorem that  \ref{thm:change of prior}, the function class of $\vecn$ is sandwiched in-between $\gapn$ and $\lcnnn_{dn}$. In the small dataset regime, the algorithm is unable to move $\vecn$ far away from the $\lcnnn_{dn}$-like regime ($\Othreed$-invariance). This behavior is reflected from Fig.~\ref{fig:vec bends learning curve}: for $m<10^4$) the {\bf\color{plot_blue} Blue Lines} ($\vecn$) is very close to the {\bf\color{plot_green} Green Lines} ($\vecn$ with $\Othreed$ spurious symmetry applied to the data), and the slopes of all learning curves are comparable. With more data, $\vecn$ is able to break the $\Othreed$ symmetry and being to perform feature learning, which moves the model away from $\lcnnn$-like regime and towards $\gapn$. We provide further empirical support for this hypothesis in Sec.~\ref{sec:measureing-ssb} by showing that better performing $\vecn$-learners are closer to $\gapn$-learners and further from the $\vecinf$-learners, and vice versa. 

\begin{figure}[t]
    \centering
    \includegraphics[width=0.23\textwidth]{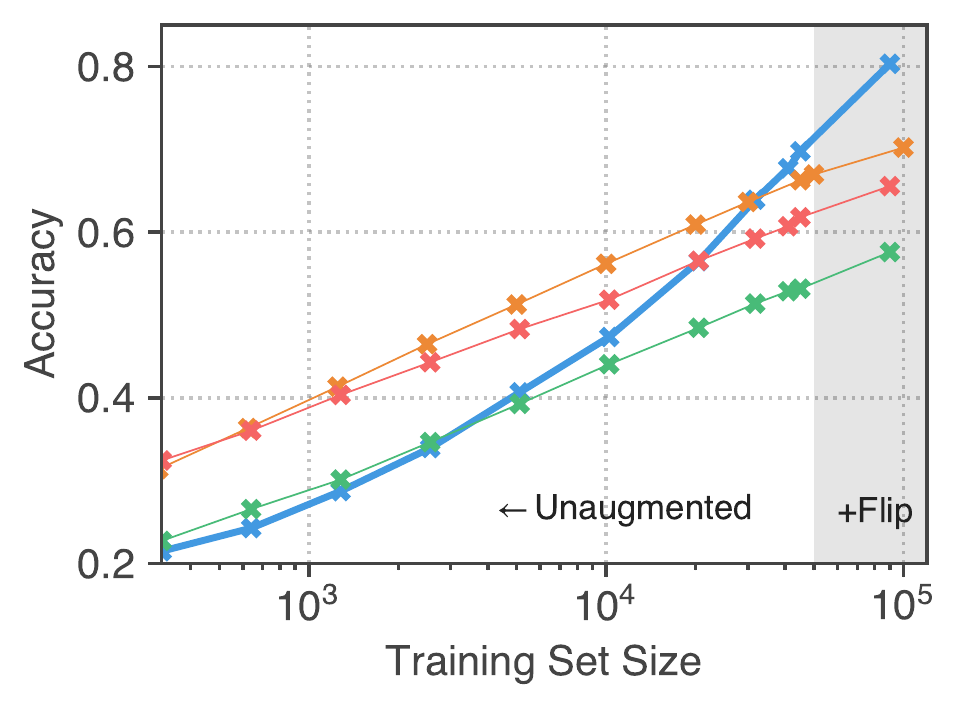}
    \includegraphics[width=0.23\textwidth]{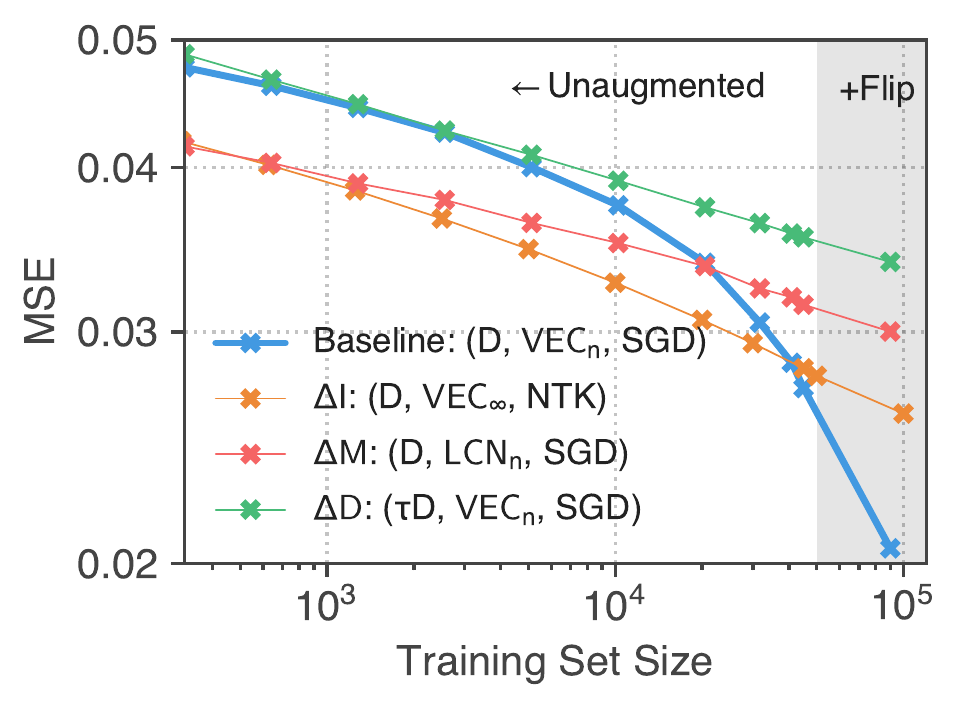}
        \caption{{\bf More Data Improves Learning Curve of $\vecn$.} 
high        Left: test accuracy. Right: MSE loss. The slope of the learning curve of the optimal baseline \dmi configuration ({\bf \color{plot_blue} blue}) increases significantly in the large-data regime. This data-improved efficiency disappears under interventions that suboptimally deform the data $\D$ ({\bf  \color{plot_green} green}), the model $\M$ ({\bf  \color{plot_red} red}), or the inference algorithm $\I$ ({\bf \color{plot_orange} orange}).}
            \label{fig:vec bends learning curve}
\end{figure}

\subsection{DIDE for ImageNet}
Larger deep learning systems can be exhibit qualitative differences from smaller ones. As such, we examine the DIDE phenomenon for SoTA models on ImageNet \citep{Deng2011Imagenet}, which has $m_{\text{ImageNet}}\sim 10^6$ training samples with size $224\times 224 \times 3$ ($=3d$). 

We subsample $m_i = 2^{-i/2}m_{\text{ImageNet}}$ images as our training set, where $i\in[12]$. 
In Fig.~\ref{fig:imagenet scaling}, we plot the learning curves for four \dmi triplets. We set $\M$ to be ResNet101 or a (small MLP-)Mixer \citep{tolstikhin2021mlp}, $\D$ to be the original (unrotated) ImageNet $\D_{\text{IN}}$ or a rotated version $\tau \D_{\text{IN}}$, for some $\tau\in\Othreed$. We keep $\I$ (SGD, see Sec.~\ref{sec:ImageNet experiment}) fixed. For each configuration, we interpolate the first/last six points (i.e.,  $i\geq6$/$i<6$) with straight lines (in the log-log plot) and compute the slope $\alpha$ (see legends in Fig.~\ref{fig:imagenet scaling}). We treat ResNet101 trained on clean images as our baseline, which is the most efficient and consistent \dmi system among the four. We observe the following. 
(1) Almost perfect power-law scaling for the baseline ({\bf \color{plot_blue}Blue Dashed Line}). (2) A cusp around $m_{i=6}$ for the remaining learning curves, which have two phases: relatively flat in the first phase and steepened in the second one. (3) Surprisingly, the slopes ($\alpha=0.49, 0.38$) of $(\text{ResNet101}, \tau\D_{\text{IN}}, \I)$ and $(\text{Mixer}, \D_{\text{IN}}, \I)$ in the second phase essentially catch up with that ($\alpha=0.41$) of the best one $(\text{ResNet101}, \D_{\text{IN}}, \I)$. 
These observations suggest that with more data the system can overcome the spurious symmetries $\Othreed$.  

Finally, to test the limit of deep learning systems in overcoming spurious $\Othreed$ symmetries, we further scale up $\M$ in the ResNet family \citep{he2016deep} and in the EfficientNet family \cite{tan2019efficientnet}. In the right panel of Fig.~\ref{fig:vec bends learning curve}, we make a scatter plot showing the accuracy of the $\Othreed$-rotated vs the original dataset. Each data point corresponds to one model.  
For the ResNet family, the top-1 accuracy gap between the rotated and the unrotated dataset drops from about 10\% (ResNet-18) to about 6\% (ResNet-200) and for the EfficientNet family, this gap drops from about 4\% (EfficientNet B0) to {\it only} about 1.\% (EfficientNet B7 \footnote{Trained by about 180 epochs.}), which is quite remarkable.  
\paragraph{Discussion of DIDE.} The change of the slopes of the learning curves suggests that the function classes on the left/right of the cusp might be qualitatively quite different. The cusp happens around $m_{i=6}\sim 2\times 10^5$, which is of the same order of $\dim(\Othreed)=3d=150528$ and $\dim(\text{O}(16^2\times 3))=294528$ (the size of the patches in the Mixer is (16, 16, 3)). This agreement suggests that to overcome spurious symmetry $\mathcal G$ (either from the models or data), at least $\sim\dim(\mathcal G)$ extra training points are needed. We also test the capability of the ResNet family in overcoming the $\text{P}(3d)$ (< $\Odthree$) symmetry (Sec.\ref{sec:learning-dyanmics-imagenet-rotation}), but the test and training accuracy remain below 35\% for all ResNet models. If we extrapolate the dimension counting argument, $\dim(\Odthree)\sim 10^{10}$ many training points may be needed to overcome the $\Odthree$ symmetries.

\begin{figure}
    \centering
    
    \includegraphics[width=0.23\textwidth]{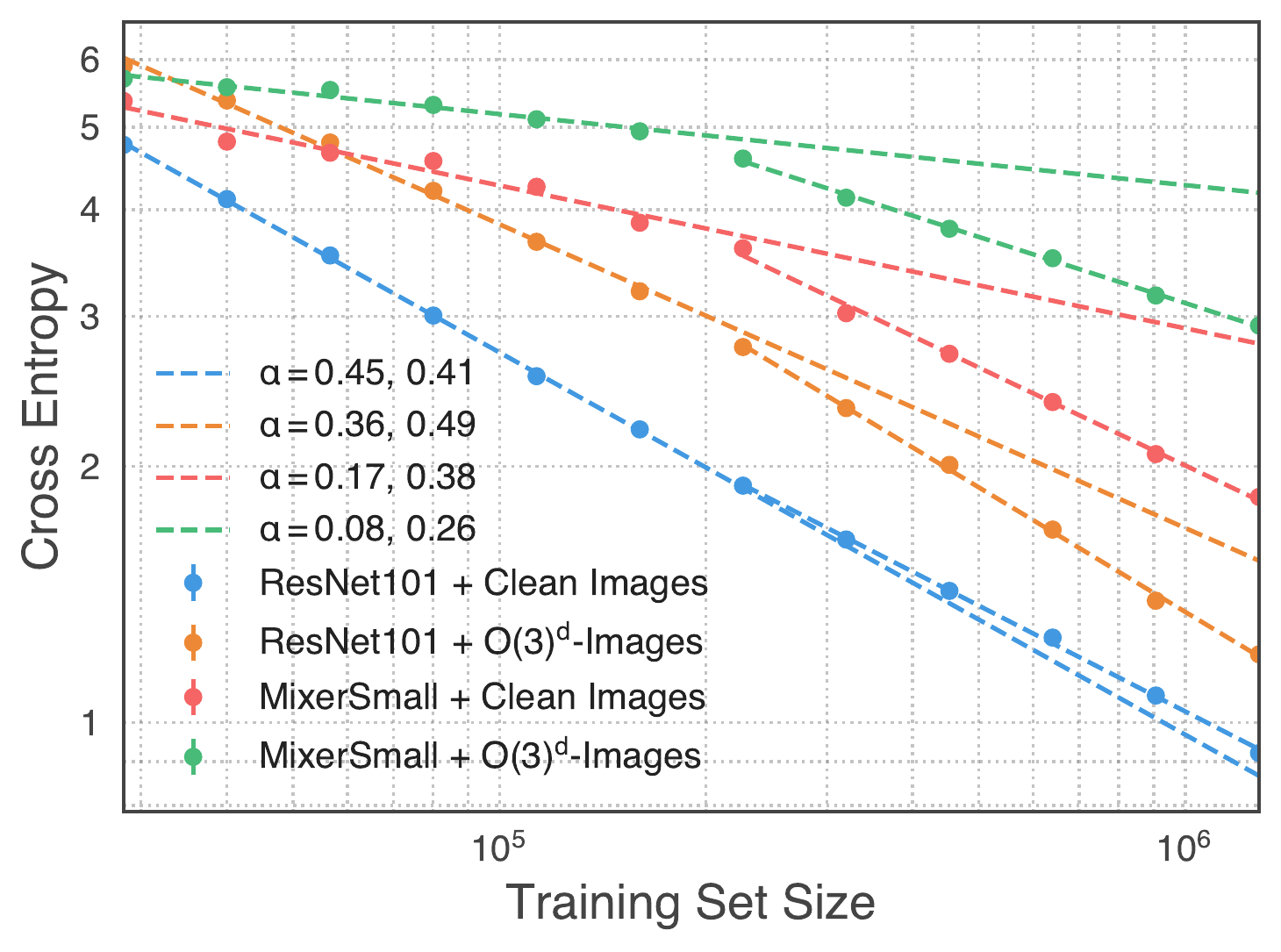}
    \includegraphics[width=0.23\textwidth]{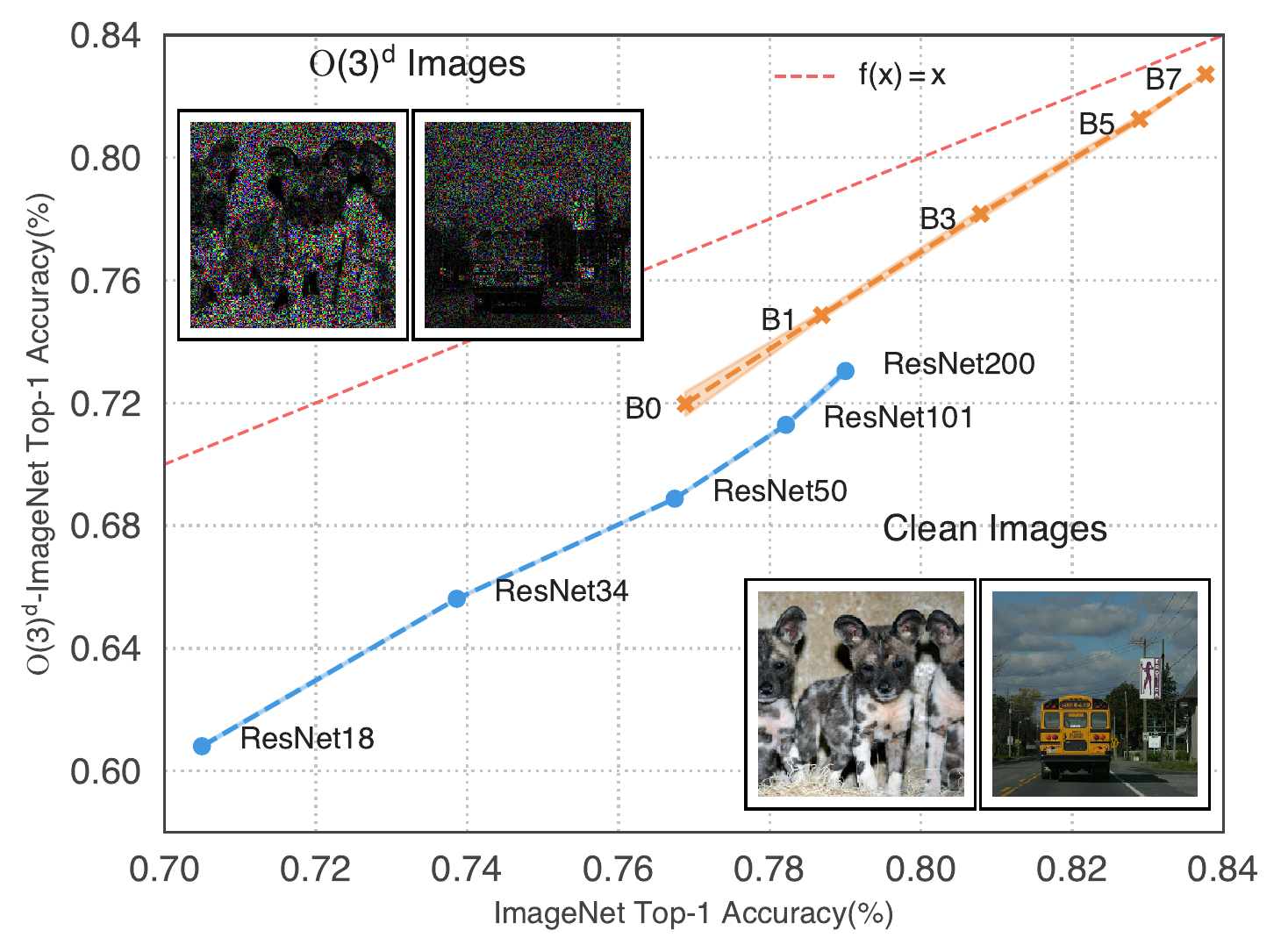}
    \caption{{\bf SoTA models overcome $\Othreed$ spurious symmetries.}
    Left: learning curves for four \dmi triplets. There is a cusp in all learning curves except the optimal configuration ({\bf \color{plot_blue} Blue curve.}) Learning efficiency significantly improves for those curves after the cusp. 
    Right: Top-1 accuracy for the ResNet and the EfficientNet families when the dataset is $\Othreed$-rotated ($x$-axis) and unrotated ($y$-axis.) SoTA models can overcome the $\Othreed$ spurious symmetries. 
    }
    \label{fig:imagenet scaling}
\end{figure}

\section{Conclusion}
We consider machine learning methods as an integrated system of data, models, and inference algorithms and study the basic symmetries of various machine learning systems \dmi. We examine the relation between the performance and the consistency of the triplet \dmi through the lens of symmetries. We find that learning is most efficient when the learning algorithm is consistent with the data distribution. Finally, we observe that, for many triplets \dmi, the slopes of the scaling law curves can improve with more data, suggesting the function class is transitioning to a new one that is dramatically more data-efficient than the one obtained from initialization. Theoretical characterization of how this transition occurs might be a crucial step to the understanding of feature learning in neural networks.

\section*{Acknowledgetment.}
We thank Jascha Sohl-dickstein, Sam Schoenholz, Jaehoon Lee, Roman Novak, and Yasaman Bahri for insightful discussion, and Sam Schoenholz, Jaehoon Lee, Roman Novak, and Mingxing Tang for engineering support. 
We are also grateful to Roman Novak and the anonymous reviewers for feedback and suggestions on an earlier draft of this work.




\begin{thebibliography}{87}
\providecommand{\natexlab}[1]{#1}
\providecommand{\url}[1]{\texttt{#1}}
\expandafter\ifx\csname urlstyle\endcsname\relax
  \providecommand{\doi}[1]{doi: #1}\else
  \providecommand{\doi}{doi: \begingroup \urlstyle{rm}\Url}\fi

\bibitem[Abadi et~al.(2016)Abadi, Barham, Chen, Chen, Davis, Dean, Devin,
  Ghemawat, Irving, Isard, et~al.]{abadi2016tensorflow}
Abadi, M., Barham, P., Chen, J., Chen, Z., Davis, A., Dean, J., Devin, M.,
  Ghemawat, S., Irving, G., Isard, M., et~al.
\newblock Tensorflow: A system for large-scale machine learning.
\newblock In \emph{12th {USENIX} Symposium on Operating Systems Design and
  Implementation ({OSDI} 16)}, 2016.

\bibitem[Adlam \& Pennington(2020)Adlam and Pennington]{adlam2020neural}
Adlam, B. and Pennington, J.
\newblock The neural tangent kernel in high dimensions: Triple descent and a
  multi-scale theory of generalization, 2020.

\bibitem[Allen-Zhu et~al.(2018)Allen-Zhu, Li, and Song]{allen2018convergence}
Allen-Zhu, Z., Li, Y., and Song, Z.
\newblock A convergence theory for deep learning via over-parameterization.
\newblock In \emph{International Conference on Machine Learning}, 2018.

\bibitem[Arora et~al.(2019)Arora, Du, Hu, Li, and Wang]{arora2019finegrained}
Arora, S., Du, S.~S., Hu, W., Li, Z., and Wang, R.
\newblock Fine-grained analysis of optimization and generalization for
  overparameterized two-layer neural networks, 2019.

\bibitem[Bach(2017)]{bach2017breaking}
Bach, F.
\newblock Breaking the curse of dimensionality with convex neural networks.
\newblock \emph{The Journal of Machine Learning Research}, 18\penalty0
  (1):\penalty0 629--681, 2017.

\bibitem[Bahri et~al.(2021)Bahri, Dyer, Kaplan, Lee, and
  Sharma]{bahri2021explaining}
Bahri, Y., Dyer, E., Kaplan, J., Lee, J., and Sharma, U.
\newblock Explaining neural scaling laws.
\newblock \emph{arXiv preprint arXiv:2102.06701}, 2021.

\bibitem[Bietti(2021)]{bietti2021approximation}
Bietti, A.
\newblock Approximation and learning with deep convolutional models: a kernel
  perspective, 2021.

\bibitem[Bordelon et~al.(2021)Bordelon, Canatar, and
  Pehlevan]{bordelon2021spectrum}
Bordelon, B., Canatar, A., and Pehlevan, C.
\newblock Spectrum dependent learning curves in kernel regression and wide
  neural networks, 2021.

\bibitem[Bradbury et~al.(2018)Bradbury, Frostig, Hawkins, Johnson, Leary,
  Maclaurin, and Wanderman-Milne]{jax2018github}
Bradbury, J., Frostig, R., Hawkins, P., Johnson, M.~J., Leary, C., Maclaurin,
  D., and Wanderman-Milne, S.
\newblock {JAX}: composable transformations of {P}ython+{N}um{P}y programs,
  2018.
\newblock URL \url{http://github.com/google/jax}.

\bibitem[Bruna \& Mallat(2013)Bruna and Mallat]{bruna2013invariant}
Bruna, J. and Mallat, S.
\newblock Invariant scattering convolution networks.
\newblock \emph{IEEE transactions on pattern analysis and machine
  intelligence}, 35\penalty0 (8):\penalty0 1872--1886, 2013.

\bibitem[Chen et~al.(2020)Chen, Dobriban, and Lee]{chen2020group}
Chen, S., Dobriban, E., and Lee, J.
\newblock A group-theoretic framework for data augmentation.
\newblock \emph{Advances in Neural Information Processing Systems},
  33:\penalty0 21321--21333, 2020.

\bibitem[Chizat \& Bach(2020)Chizat and Bach]{chizat20a}
Chizat, L. and Bach, F.
\newblock Implicit bias of gradient descent for wide two-layer neural networks
  trained with the logistic loss.
\newblock In Abernethy, J. and Agarwal, S. (eds.), \emph{Proceedings of Thirty
  Third Conference on Learning Theory}, volume 125 of \emph{Proceedings of
  Machine Learning Research}, pp.\  1305--1338. PMLR, 09--12 Jul 2020.
\newblock URL \url{http://proceedings.mlr.press/v125/chizat20a.html}.

\bibitem[Chizat et~al.(2019)Chizat, Oyallon, and Bach]{chizat2019lazy}
Chizat, L., Oyallon, E., and Bach, F.
\newblock On lazy training in differentiable programming.
\newblock \emph{Advances in Neural Information Processing Systems},
  32:\penalty0 2937--2947, 2019.

\bibitem[Cohen \& Welling(2016)Cohen and Welling]{cohen2016group}
Cohen, T. and Welling, M.
\newblock Group equivariant convolutional networks.
\newblock In \emph{International conference on machine learning}, pp.\
  2990--2999. PMLR, 2016.

\bibitem[Daniely(2017)]{daniely2017sgd}
Daniely, A.
\newblock {SGD} learns the conjugate kernel class of the network.
\newblock In \emph{Advances in Neural Information Processing Systems}, pp.\
  2422--2430, 2017.

\bibitem[Daniely et~al.(2016)Daniely, Frostig, and Singer]{daniely2016toward}
Daniely, A., Frostig, R., and Singer, Y.
\newblock Toward deeper understanding of neural networks: The power of
  initialization and a dual view on expressivity.
\newblock In \emph{Advances In Neural Information Processing Systems}, pp.\
  2253--2261, 2016.

\bibitem[Deng et~al.(2009)Deng, Dong, Socher, Li, Li, and
  Fei-Fei]{Deng2011Imagenet}
Deng, J., Dong, W., Socher, R., Li, L.-J., Li, K., and Fei-Fei, L.
\newblock Imagenet: A large-scale hierarchical image database.
\newblock In \emph{2009 IEEE Conference on Computer Vision and Pattern
  Recognition}, pp.\  248--255, 2009.
\newblock \doi{10.1109/CVPR.2009.5206848}.

\bibitem[Devlin et~al.(2018)Devlin, Chang, Lee, and Toutanova]{devlin2018bert}
Devlin, J., Chang, M.-W., Lee, K., and Toutanova, K.
\newblock Bert: Pre-training of deep bidirectional transformers for language
  understanding.
\newblock \emph{arXiv preprint arXiv:1810.04805}, 2018.

\bibitem[Dosovitskiy et~al.(2020)Dosovitskiy, Beyer, Kolesnikov, Weissenborn,
  Zhai, Unterthiner, Dehghani, Minderer, Heigold, Gelly,
  et~al.]{dosovitskiy2020image}
Dosovitskiy, A., Beyer, L., Kolesnikov, A., Weissenborn, D., Zhai, X.,
  Unterthiner, T., Dehghani, M., Minderer, M., Heigold, G., Gelly, S., et~al.
\newblock An image is worth 16x16 words: Transformers for image recognition at
  scale.
\newblock \emph{arXiv preprint arXiv:2010.11929}, 2020.

\bibitem[Du et~al.(2018)Du, Lee, Li, Wang, and Zhai]{du2018global}
Du, S.~S., Lee, J.~D., Li, H., Wang, L., and Zhai, X.
\newblock Gradient descent finds global minima of deep neural networks.
\newblock \emph{arXiv preprint arXiv:1811.03804}, 2018.

\bibitem[Favero et~al.(2021)Favero, Cagnetta, and Wyart]{favero2021locality}
Favero, A., Cagnetta, F., and Wyart, M.
\newblock Locality defeats the curse of dimensionality in convolutional
  teacher-student scenarios, 2021.

\bibitem[Fukushima(1975)]{fukushima1975cognitron}
Fukushima, K.
\newblock Cognitron: A self-organizing multilayered neural network.
\newblock \emph{Biological cybernetics}, 20\penalty0 (3-4):\penalty0 121--136,
  1975.

\bibitem[Garriga-Alonso et~al.(2019)Garriga-Alonso, Aitchison, and
  Rasmussen]{garriga2018deep}
Garriga-Alonso, A., Aitchison, L., and Rasmussen, C.~E.
\newblock Deep convolutional networks as shallow gaussian processes.
\newblock In \emph{International Conference on Learning Representations}, 2019.

\bibitem[Goldt et~al.(2020)Goldt, M{\'e}zard, Krzakala, and
  Zdeborov{\'a}]{goldt2020modeling}
Goldt, S., M{\'e}zard, M., Krzakala, F., and Zdeborov{\'a}, L.
\newblock Modeling the influence of data structure on learning in neural
  networks: The hidden manifold model.
\newblock \emph{Physical Review X}, 10\penalty0 (4):\penalty0 041044, 2020.

\bibitem[Gunasekar et~al.(2018)Gunasekar, Woodworth, Bhojanapalli, Neyshabur,
  and Srebro]{gunasekar2018implicit}
Gunasekar, S., Woodworth, B., Bhojanapalli, S., Neyshabur, B., and Srebro, N.
\newblock Implicit regularization in matrix factorization.
\newblock In \emph{2018 Information Theory and Applications Workshop (ITA)},
  pp.\  1--10. IEEE, 2018.

\bibitem[He et~al.(2016)He, Zhang, Ren, and Sun]{he2016deep}
He, K., Zhang, X., Ren, S., and Sun, J.
\newblock Deep residual learning for image recognition.
\newblock In \emph{Proceedings of the IEEE conference on computer vision and
  pattern recognition}, pp.\  770--778, 2016.

\bibitem[He et~al.(2021)He, Chen, Xie, Li, Doll{\'a}r, and
  Girshick]{he2021masked}
He, K., Chen, X., Xie, S., Li, Y., Doll{\'a}r, P., and Girshick, R.
\newblock Masked autoencoders are scalable vision learners.
\newblock \emph{arXiv preprint arXiv:2111.06377}, 2021.

\bibitem[Heek et~al.(2020)Heek, Levskaya, Oliver, Ritter, Rondepierre, Steiner,
  and van {Z}ee]{flax2020github}
Heek, J., Levskaya, A., Oliver, A., Ritter, M., Rondepierre, B., Steiner, A.,
  and van {Z}ee, M.
\newblock {F}lax: A neural network library and ecosystem for {JAX}, 2020.
\newblock URL \url{http://github.com/google/flax}.

\bibitem[Hoffmann et~al.(2022)Hoffmann, Borgeaud, Mensch, Buchatskaya, Cai,
  Rutherford, Casas, Hendricks, Welbl, Clark, et~al.]{hoffmann2022training}
Hoffmann, J., Borgeaud, S., Mensch, A., Buchatskaya, E., Cai, T., Rutherford,
  E., Casas, D. d.~L., Hendricks, L.~A., Welbl, J., Clark, A., et~al.
\newblock Training compute-optimal large language models.
\newblock \emph{arXiv preprint arXiv:2203.15556}, 2022.

\bibitem[Hron et~al.(2020{\natexlab{a}})Hron, Bahri, Novak, Pennington, and
  Sohl-Dickstein]{hron2020exact}
Hron, J., Bahri, Y., Novak, R., Pennington, J., and Sohl-Dickstein, J.
\newblock Exact posterior distributions of wide bayesian neural networks.
\newblock \emph{arXiv preprint arXiv:2006.10541}, 2020{\natexlab{a}}.

\bibitem[Hron et~al.(2020{\natexlab{b}})Hron, Bahri, Sohl-Dickstein, and
  Novak]{hron2020infinite}
Hron, J., Bahri, Y., Sohl-Dickstein, J., and Novak, R.
\newblock Infinite attention: Nngp and ntk for deep attention networks,
  2020{\natexlab{b}}.

\bibitem[Hu et~al.(2020)Hu, Xiao, Adlam, and Pennington]{hu2020surprising}
Hu, W., Xiao, L., Adlam, B., and Pennington, J.
\newblock The surprising simplicity of the early-time learning dynamics of
  neural networks.
\newblock \emph{arXiv preprint arXiv:2006.14599}, 2020.

\bibitem[Ingrosso \& Goldt(2022)Ingrosso and Goldt]{ingrosso2022data}
Ingrosso, A. and Goldt, S.
\newblock Data-driven emergence of convolutional structure in neural networks.
\newblock \emph{arXiv preprint arXiv:2202.00565}, 2022.

\bibitem[Jacot et~al.(2018{\natexlab{a}})Jacot, Gabriel, and
  Hongler]{Jacot2018ntk}
Jacot, A., Gabriel, F., and Hongler, C.
\newblock Neural tangent kernel: Convergence and generalization in neural
  networks.
\newblock In \emph{Advances in Neural Information Processing Systems},
  2018{\natexlab{a}}.

\bibitem[Jacot et~al.(2018{\natexlab{b}})Jacot, Gabriel, and
  Hongler]{jacot2018neural}
Jacot, A., Gabriel, F., and Hongler, C.
\newblock Neural tangent kernel: Convergence and generalization in neural
  networks.
\newblock \emph{arXiv preprint arXiv:1806.07572}, 2018{\natexlab{b}}.

\bibitem[Ji \& Telgarsky(2019{\natexlab{a}})Ji and Telgarsky]{ji2019implicit}
Ji, Z. and Telgarsky, M.
\newblock The implicit bias of gradient descent on nonseparable data.
\newblock In \emph{Conference on Learning Theory}, pp.\  1772--1798,
  2019{\natexlab{a}}.

\bibitem[Ji \& Telgarsky(2019{\natexlab{b}})Ji and Telgarsky]{ji2019gradient}
Ji, Z. and Telgarsky, M.~J.
\newblock Gradient descent aligns the layers of deep linear networks.
\newblock In \emph{7th International Conference on Learning Representations,
  ICLR 2019}, 2019{\natexlab{b}}.

\bibitem[Kaplan et~al.(2020)Kaplan, McCandlish, Henighan, Brown, Chess, Child,
  Gray, Radford, Wu, and Amodei]{kaplan2020scaling}
Kaplan, J., McCandlish, S., Henighan, T., Brown, T.~B., Chess, B., Child, R.,
  Gray, S., Radford, A., Wu, J., and Amodei, D.
\newblock Scaling laws for neural language models.
\newblock \emph{arXiv preprint arXiv:2001.08361}, 2020.

\bibitem[Krizhevsky et~al.(2009)Krizhevsky, Hinton,
  et~al.]{krizhevsky2009learning}
Krizhevsky, A., Hinton, G., et~al.
\newblock Learning multiple layers of features from tiny images.
\newblock 2009.

\bibitem[Krizhevsky et~al.(2012)Krizhevsky, Sutskever, and
  Hinton]{krizhevsky2012imagenet}
Krizhevsky, A., Sutskever, I., and Hinton, G.~E.
\newblock Imagenet classification with deep convolutional neural networks.
\newblock In \emph{Advances in neural information processing systems}, pp.\
  1097--1105, 2012.

\bibitem[Lecun(1989)]{lecun1989generalization}
Lecun, Y.
\newblock Generalization and network design strategies.
\newblock In \emph{Connectionism in perspective}. Elsevier, 1989.

\bibitem[LeCun et~al.(1998)LeCun, Bottou, Bengio, and
  Haffner]{lecun1998gradient}
LeCun, Y., Bottou, L., Bengio, Y., and Haffner, P.
\newblock Gradient-based learning applied to document recognition.
\newblock \emph{Proceedings of the IEEE}, 86\penalty0 (11):\penalty0
  2278--2324, 1998.

\bibitem[Lee et~al.(2018)Lee, Sohl-dickstein, Pennington, Novak, Schoenholz,
  and Bahri]{lee2018deep}
Lee, J., Sohl-dickstein, J., Pennington, J., Novak, R., Schoenholz, S., and
  Bahri, Y.
\newblock Deep neural networks as gaussian processes.
\newblock In \emph{International Conference on Learning Representations}, 2018.
\newblock URL \url{https://openreview.net/forum?id=B1EA-M-0Z}.

\bibitem[Lee et~al.(2019)Lee, Xiao, Schoenholz, Bahri, Novak, Sohl-Dickstein,
  and Pennington]{lee2019wide}
Lee, J., Xiao, L., Schoenholz, S.~S., Bahri, Y., Novak, R., Sohl-Dickstein, J.,
  and Pennington, J.
\newblock Wide neural networks of any depth evolve as linear models under
  gradient descent.
\newblock In \emph{Advances in Neural Information Processing Systems}, 2019.

\bibitem[Lee et~al.(2020)Lee, Schoenholz, Pennington, Adlam, Xiao, Novak, and
  Sohl-Dickstein]{lee2020finite}
Lee, J., Schoenholz, S.~S., Pennington, J., Adlam, B., Xiao, L., Novak, R., and
  Sohl-Dickstein, J.
\newblock Finite versus infinite neural networks: an empirical study.
\newblock \emph{arXiv preprint arXiv:2007.15801}, 2020.

\bibitem[Li et~al.(2018)Li, Farkhoor, Liu, and Yosinski]{li2018measuring}
Li, C., Farkhoor, H., Liu, R., and Yosinski, J.
\newblock Measuring the intrinsic dimension of objective landscapes.
\newblock \emph{arXiv preprint arXiv:1804.08838}, 2018.

\bibitem[Li et~al.(2020)Li, Zhang, and Arora]{li2020convolutional}
Li, Z., Zhang, Y., and Arora, S.
\newblock Why are convolutional nets more sample-efficient than fully-connected
  nets?
\newblock \emph{arXiv preprint arXiv:2010.08515}, 2020.

\bibitem[Lyu \& Li(2020)Lyu and Li]{Lyu2020Gradient}
Lyu, K. and Li, J.
\newblock Gradient descent maximizes the margin of homogeneous neural networks.
\newblock In \emph{International Conference on Learning Representations}, 2020.
\newblock URL \url{https://openreview.net/forum?id=SJeLIgBKPS}.

\bibitem[Matthews et~al.(2018)Matthews, Hron, Rowland, Turner, and
  Ghahramani]{matthews2018}
Matthews, A., Hron, J., Rowland, M., Turner, R.~E., and Ghahramani, Z.
\newblock Gaussian process behaviour in wide deep neural networks.
\newblock In \emph{International Conference on Learning Representations}, 2018.

\bibitem[Mei et~al.(2018)Mei, Montanari, and Nguyen]{mei2018mean}
Mei, S., Montanari, A., and Nguyen, P.-M.
\newblock A mean field view of the landscape of two-layer neural networks.
\newblock \emph{Proceedings of the National Academy of Sciences}, 115\penalty0
  (33):\penalty0 E7665--E7671, 2018.

\bibitem[Mei et~al.(2021)Mei, Misiakiewicz, and Montanari]{mei2021learning}
Mei, S., Misiakiewicz, T., and Montanari, A.
\newblock Learning with invariances in random features and kernel models, 2021.

\bibitem[Nakkiran et~al.(2019)Nakkiran, Kaplun, Kalimeris, Yang, Edelman,
  Zhang, and Barak]{nakkiran2019sgd}
Nakkiran, P., Kaplun, G., Kalimeris, D., Yang, T., Edelman, B.~L., Zhang, F.,
  and Barak, B.
\newblock Sgd on neural networks learns functions of increasing complexity.
\newblock \emph{arXiv preprint arXiv:1905.11604}, 2019.

\bibitem[Neal(1994)]{neal}
Neal, R.~M.
\newblock Priors for infinite networks (tech. rep. no. crg-tr-94-1).
\newblock \emph{University of Toronto}, 1994.

\bibitem[Neyshabur(2017)]{neyshabur2017implicit}
Neyshabur, B.
\newblock Implicit regularization in deep learning.
\newblock \emph{arXiv preprint arXiv:1709.01953}, 2017.

\bibitem[Neyshabur(2020)]{neyshabur2020towards}
Neyshabur, B.
\newblock Towards learning convolutions from scratch.
\newblock \emph{arXiv preprint arXiv:2007.13657}, 2020.

\bibitem[Novak et~al.(2019{\natexlab{a}})Novak, Xiao, Hron, Lee, Alemi,
  Sohl-Dickstein, and Schoenholz]{novak2019neural}
Novak, R., Xiao, L., Hron, J., Lee, J., Alemi, A.~A., Sohl-Dickstein, J., and
  Schoenholz, S.~S.
\newblock Neural tangents: Fast and easy infinite neural networks in python.
\newblock \emph{arXiv preprint arXiv:1912.02803}, 2019{\natexlab{a}}.

\bibitem[Novak et~al.(2019{\natexlab{b}})Novak, Xiao, Lee, Bahri, Yang, Hron,
  Abolafia, Pennington, and Sohl-Dickstein]{novak2018bayesian}
Novak, R., Xiao, L., Lee, J., Bahri, Y., Yang, G., Hron, J., Abolafia, D.~A.,
  Pennington, J., and Sohl-Dickstein, J.
\newblock Bayesian deep convolutional networks with many channels are gaussian
  processes.
\newblock In \emph{International Conference on Learning Representations},
  2019{\natexlab{b}}.

\bibitem[Novak et~al.(2020)Novak, Xiao, Hron, Lee, Alemi, Sohl-Dickstein, and
  Schoenholz]{neuraltangents2020}
Novak, R., Xiao, L., Hron, J., Lee, J., Alemi, A.~A., Sohl-Dickstein, J., and
  Schoenholz, S.~S.
\newblock Neural tangents: Fast and easy infinite neural networks in python.
\newblock In \emph{International Conference on Learning Representations}, 2020.
\newblock URL \url{https://github.com/google/neural-tangents}.

\bibitem[Petrini et~al.(2021)Petrini, Favero, Geiger, and
  Wyart]{petrini2021relative}
Petrini, L., Favero, A., Geiger, M., and Wyart, M.
\newblock Relative stability toward diffeomorphisms indicates performance in
  deep nets.
\newblock \emph{Advances in Neural Information Processing Systems}, 34, 2021.

\bibitem[Poole et~al.(2016)Poole, Lahiri, Raghu, Sohl-Dickstein, and
  Ganguli]{poole2016exponential}
Poole, B., Lahiri, S., Raghu, M., Sohl-Dickstein, J., and Ganguli, S.
\newblock Exponential expressivity in deep neural networks through transient
  chaos.
\newblock In \emph{Advances In Neural Information Processing Systems}, 2016.

\bibitem[Pope et~al.(2021)Pope, Zhu, Abdelkader, Goldblum, and
  Goldstein]{pope2021intrinsic}
Pope, P., Zhu, C., Abdelkader, A., Goldblum, M., and Goldstein, T.
\newblock The intrinsic dimension of images and its impact on learning.
\newblock \emph{arXiv preprint arXiv:2104.08894}, 2021.

\bibitem[Rahaman et~al.(2019)Rahaman, Baratin, Arpit, Draxler, Lin, Hamprecht,
  Bengio, and Courville]{rahaman2019spectral}
Rahaman, N., Baratin, A., Arpit, D., Draxler, F., Lin, M., Hamprecht, F.,
  Bengio, Y., and Courville, A.
\newblock On the spectral bias of neural networks.
\newblock In \emph{International Conference on Machine Learning}, pp.\
  5301--5310. PMLR, 2019.

\bibitem[Schoenholz et~al.(2017)Schoenholz, Gilmer, Ganguli, and
  Sohl-Dickstein]{schoenholz2016}
Schoenholz, S.~S., Gilmer, J., Ganguli, S., and Sohl-Dickstein, J.
\newblock Deep information propagation.
\newblock \emph{International Conference on Learning Representations}, 2017.

\bibitem[Senior et~al.(2020)Senior, Evans, Jumper, Kirkpatrick, Sifre, Green,
  Qin, {\v{Z}}{\'\i}dek, Nelson, Bridgland, et~al.]{senior2020improved}
Senior, A.~W., Evans, R., Jumper, J., Kirkpatrick, J., Sifre, L., Green, T.,
  Qin, C., {\v{Z}}{\'\i}dek, A., Nelson, A.~W., Bridgland, A., et~al.
\newblock Improved protein structure prediction using potentials from deep
  learning.
\newblock \emph{Nature}, 577\penalty0 (7792):\penalty0 706--710, 2020.

\bibitem[Shalev-Shwartz \& Ben-David(2014)Shalev-Shwartz and
  Ben-David]{shalev2014understanding}
Shalev-Shwartz, S. and Ben-David, S.
\newblock \emph{Understanding machine learning: From theory to algorithms}.
\newblock Cambridge university press, 2014.

\bibitem[Silver et~al.(2016)Silver, Huang, Maddison, Guez, Sifre, Van
  Den~Driessche, Schrittwieser, Antonoglou, Panneershelvam, Lanctot,
  et~al.]{silver2016mastering}
Silver, D., Huang, A., Maddison, C.~J., Guez, A., Sifre, L., Van Den~Driessche,
  G., Schrittwieser, J., Antonoglou, I., Panneershelvam, V., Lanctot, M.,
  et~al.
\newblock Mastering the game of go with deep neural networks and tree search.
\newblock \emph{nature}, 529\penalty0 (7587):\penalty0 484--489, 2016.

\bibitem[Sohl-Dickstein et~al.(2020)Sohl-Dickstein, Novak, Schoenholz, and
  Lee]{sohl2020infinite}
Sohl-Dickstein, J., Novak, R., Schoenholz, S.~S., and Lee, J.
\newblock On the infinite width limit of neural networks with a standard
  parameterization.
\newblock \emph{arXiv preprint arXiv:2001.07301}, 2020.

\bibitem[Song et~al.(2018)Song, Montanari, and Nguyen]{song2018mean}
Song, M., Montanari, A., and Nguyen, P.
\newblock A mean field view of the landscape of two-layers neural networks.
\newblock \emph{Proceedings of the National Academy of Sciences}, 115:\penalty0
  E7665--E7671, 2018.

\bibitem[Soudry et~al.(2018)Soudry, Hoffer, Nacson, Gunasekar, and
  Srebro]{soudry2018implicit}
Soudry, D., Hoffer, E., Nacson, M.~S., Gunasekar, S., and Srebro, N.
\newblock The implicit bias of gradient descent on separable data.
\newblock \emph{Journal of Machine Learning Research}, 19\penalty0 (70), 2018.

\bibitem[Su \& Yang(2019)Su and Yang]{su2019learning}
Su, L. and Yang, P.
\newblock On learning over-parameterized neural networks: A functional
  approximation perspective.
\newblock \emph{arXiv preprint arXiv:1905.10826}, 2019.

\bibitem[Tan \& Le(2019)Tan and Le]{tan2019efficientnet}
Tan, M. and Le, Q.
\newblock Efficientnet: Rethinking model scaling for convolutional neural
  networks.
\newblock In \emph{International Conference on Machine Learning}, pp.\
  6105--6114. PMLR, 2019.

\bibitem[Tolstikhin et~al.(2021)Tolstikhin, Houlsby, Kolesnikov, Beyer, Zhai,
  Unterthiner, Yung, Keysers, Uszkoreit, Lucic, et~al.]{tolstikhin2021mlp}
Tolstikhin, I., Houlsby, N., Kolesnikov, A., Beyer, L., Zhai, X., Unterthiner,
  T., Yung, J., Keysers, D., Uszkoreit, J., Lucic, M., et~al.
\newblock Mlp-mixer: An all-mlp architecture for vision.
\newblock \emph{arXiv preprint arXiv:2105.01601}, 2021.

\bibitem[Vaswani et~al.(2017)Vaswani, Shazeer, Parmar, Uszkoreit, Jones, Gomez,
  Kaiser, and Polosukhin]{vaswani2017attention}
Vaswani, A., Shazeer, N., Parmar, N., Uszkoreit, J., Jones, L., Gomez, A.~N.,
  Kaiser, L., and Polosukhin, I.
\newblock Attention is all you need.
\newblock \emph{arXiv preprint arXiv:1706.03762}, 2017.

\bibitem[von Luxburg \& Bousquet(2004)von Luxburg and
  Bousquet]{von2004distance}
von Luxburg, U. and Bousquet, O.
\newblock Distance-based classification with lipschitz functions.
\newblock \emph{J. Mach. Learn. Res.}, 5:\penalty0 669--695, 2004.

\bibitem[Wadia et~al.(2020)Wadia, Duckworth, Schoenholz, Dyer, and
  Sohl-Dickstein]{wadia2020whitening}
Wadia, N.~S., Duckworth, D., Schoenholz, S.~S., Dyer, E., and Sohl-Dickstein,
  J.
\newblock Whitening and second order optimization both destroy information
  about the dataset, and can make generalization impossible.
\newblock \emph{arxiv preprint arXiv:2008.07545}, 2020.

\bibitem[Xiao(2021)]{xiao2021eigenspace}
Xiao, L.
\newblock Eigenspace restructuring: a principle of space and frequency in
  neural networks, \emph{COLT}, 2022.

\bibitem[Xiao et~al.(2018{\natexlab{a}})Xiao, Bahri, Sohl-Dickstein,
  Schoenholz, and Pennington]{xiao18a}
Xiao, L., Bahri, Y., Sohl-Dickstein, J., Schoenholz, S., and Pennington, J.
\newblock Dynamical isometry and a mean field theory of {CNN}s: How to train
  10,000-layer vanilla convolutional neural networks.
\newblock In \emph{International Conference on Machine Learning},
  2018{\natexlab{a}}.

\bibitem[Xiao et~al.(2018{\natexlab{b}})Xiao, Bahri, Sohl-Dickstein,
  Schoenholz, and Pennington]{xiao2018dynamical}
Xiao, L., Bahri, Y., Sohl-Dickstein, J., Schoenholz, S., and Pennington, J.
\newblock Dynamical isometry and a mean field theory of cnns: How to train
  10,000-layer vanilla convolutional neural networks.
\newblock In \emph{International Conference on Machine Learning}, pp.\
  5393--5402, 2018{\natexlab{b}}.

\bibitem[Xu(2018)]{xu2018understanding}
Xu, Z.~J.
\newblock Understanding training and generalization in deep learning by fourier
  analysis.
\newblock \emph{arXiv preprint arXiv:1808.04295}, 2018.

\bibitem[Xu et~al.(2019)Xu, Zhang, Luo, Xiao, and Ma]{xu2019frequency}
Xu, Z.-Q.~J., Zhang, Y., Luo, T., Xiao, Y., and Ma, Z.
\newblock Frequency principle: Fourier analysis sheds light on deep neural
  networks.
\newblock \emph{arXiv preprint arXiv:1901.06523}, 2019.

\bibitem[Yang(2019)]{yang2019scaling}
Yang, G.
\newblock Scaling limits of wide neural networks with weight sharing: Gaussian
  process behavior, gradient independence, and neural tangent kernel
  derivation.
\newblock \emph{arXiv preprint arXiv:1902.04760}, 2019.

\bibitem[Yang \& Hu(2020)Yang and Hu]{yang2020feature}
Yang, G. and Hu, E.~J.
\newblock Feature learning in infinite-width neural networks.
\newblock \emph{arXiv preprint arXiv:2011.14522}, 2020.

\bibitem[Yang \& Salman(2019)Yang and Salman]{yang2019fine}
Yang, G. and Salman, H.
\newblock A fine-grained spectral perspective on neural networks.
\newblock \emph{arXiv preprint arXiv:1907.10599}, 2019.

\bibitem[Yang et~al.(2022)Yang, Santacroce, and Hu]{yang2022efficient}
Yang, G., Santacroce, M., and Hu, E.~J.
\newblock Efficient computation of deep nonlinear infinite-width neural
  networks that learn features.
\newblock In \emph{International Conference on Learning Representations}, 2022.
\newblock URL \url{https://openreview.net/forum?id=tUMr0Iox8XW}.

\bibitem[Zaheer et~al.(2017)Zaheer, Kottur, Ravanbakhsh, Poczos, Salakhutdinov,
  and Smola]{zaheer2017deep}
Zaheer, M., Kottur, S., Ravanbakhsh, S., Poczos, B., Salakhutdinov, R.~R., and
  Smola, A.~J.
\newblock Deep sets.
\newblock \emph{Advances in neural information processing systems}, 30, 2017.

\bibitem[Zhai et~al.(2022)Zhai, Kolesnikov, Houlsby, and
  Beyer]{zhai2022scaling}
Zhai, X., Kolesnikov, A., Houlsby, N., and Beyer, L.
\newblock Scaling vision transformers.
\newblock In \emph{Proceedings of the IEEE/CVF Conference on Computer Vision
  and Pattern Recognition}, pp.\  12104--12113, 2022.

\bibitem[Zou et~al.(2020)Zou, Cao, Zhou, and Gu]{zou2020gradient}
Zou, D., Cao, Y., Zhou, D., and Gu, Q.
\newblock Gradient descent optimizes over-parameterized deep relu networks.
\newblock \emph{Machine Learning}, 109\penalty0 (3):\penalty0 467--492, 2020.

\end{thebibliography}
\bibliographystyle{icml2022}

\newpage
\appendix
\onecolumn

\normalsize
\onecolumn
\clearpage
\appendix

\begin{center}
\textbf{\large Supplementary Material}
\end{center}

\setcounter{equation}{0}
\setcounter{figure}{0}
\setcounter{table}{0}
\setcounter{page}{1}
\setcounter{section}{0}

\renewcommand{\theequation}{S\arabic{equation}}
\renewcommand{\thefigure}{S\arabic{figure}}
\renewcommand{\thetable}{S\arabic{table}}

\section{Glossary}
We use the following abbreviations in this work:

\begin{itemize}
    \item $\REG$: Adding L2 regularization.  
    \item $\LR$: Using a large learning rate. 
    \item $\fcnn$: Fully-connected networks with width $n$. 
    \item $\fcn$: Infinite width $\fcnn$. 
    \item $\vecn$: Convnet with width $n$ and a flattening readout layer. 
    \item $\vecinf$: Infinite width $\vecn$. 
    \item $\lcnn$: Locally-connected network with width $n$. 
    \item $\lcninf$: Infinite width $\lcnn$, which is the samme as  $\vecinf$. 
    \item $\gapn$: Convnet with width $n$ and a global average readout layer. 
    \item $\gap$: Infinite width $\gapn$. 
    \item $\lapn^k$: Similar to $\gapn$, except the readout layer is a $(k, k)$ average pooling. 
    \item $\lapinf^k$: Infinite width $\lapn^k$. 
\end{itemize}
\section{Experimental details}\label{sec:experimental details}
We use the Neural Tangents (NT) library \cite{neuraltangents2020} 
built on top of JAX \cite{jax2018github}
for all CIFAR10 experiments, and the ImageNet codebase from FLAX\footnote{\href{https://github.com/google/flax/blob/main/examples/imagenet/README.md}{https://github.com/google/flax/blob/main/examples/imagenet/README.md}}\citep{flax2020github} for ResNets and Mixer experiments on ImageNet.  

\subsection{Cifar10 Experiments}
The experimental setup is almost the same as in \cite{lee2020finite}. 
\paragraph{Architectures.} For $\fcnn$, $\lcnn$, $\vecn$ and $\gapn$, the number of hidden layers are 8 and the widths (number of channels) are $n=1024$, $32$, $128$ and $128$, resp. For all CIFAR10 experiments, we only use Relu as the activation function. We use NTK parameterization and the variances of initialization are chosen to be $\sigma_\omega^2=2$ and $\sigma_b^2=0.01$ for the weights and biases, resp.  

\paragraph{Training Details.} 
We use MSE as our loss function, which is defined to be 
\begin{align}
    L(\theta; D_{\text{mini}}) = \frac{1}{2\times K|D_{\text{mini}}|} \sum_{(x_i, y_i)\in D_{\text{mini}}} | f_{\theta}(x_i) - y_i|^2 + \frac {\lambda}{2}\|\theta\|_2^2
\end{align}
where $D_{\text{mini}}$ is a mini-batch with $|D_{\text{mini}}|=40$ and $K=10$ is the number of classes. The regularization is set to be $\lambda=10^{-7}$ if $L^2$ regularization ($\REG$) is applied otherwise 0. SGD + Momentum ($\text{mass}=0.9$) is used for all experiments. The learning rate is set to be $\eta = c\eta_0$, where $c=8$ if using a larger learning rate ($\LR$) and 1 otherwise. Here $\eta_0$\footnote{We use \href{https://github.com/google/neural-tangents/blob/main/neural_tangents/predict.py}{max\_learning\_rate}
function from the Neural Tangents library to estimate $\eta_0$.} is estimated by $\frac 2 {\lambda_{max}}$, where $\lambda_{max}$ is the largest eigenvalue of the finite-width NTK (estimated by MC sampling). With $c=8$, we are about a factor of $2\sim 4$ smaller than the maximal feasible learning rate observed empirically.

We use 45k images as the training set and reserve the remaining 5k as the validation set. All finite width experiments are initially trained for at least $10^6$ steps (about $900$ epochs), but will be early-stopped if the training accuracy reaches 100\% with fewer steps. Among the successful runs (with training accuracy $\geq 95\%$), we pick the highest test accuracy along training and average them over 5 random runs (if all runs succeed.)  

\subsection{ImageNet Experiments}\label{sec:ImageNet experiment}
We use the ImageNet codebase from FLAX\footnote{\href{https://github.com/google/flax/blob/main/examples/imagenet/README.md}{https://github.com/google/flax/blob/main/examples/imagenet/README.md}}\citep{flax2020github} for our ResNets experiments. 
We adopt most of the training configurations except change the number of epochs to 150. Note that we also rotate (see Sec.~\ref{sec:data-transformation.}) and/or subset the dataset when needed.  

The model for the Mixer experiments on ImageNet is adopted from Sec.E in \cite{tolstikhin2021mlp}. The training configurations are identical to that of the ResNets above. The EfficientNet family models are trained using Tensorflow \citep{abadi2016tensorflow}, which are adopted from \href{https://github.com/tensorflow/tpu/tree/master/models/official/efficientnet}{https://github.com/tensorflow/tpu/tree/master/models/official/efficientnet}. 

\subsection{Data Transformation.}\label{sec:data-transformation.}
Let $\D= \{(x_i, y_i), i\in [m_{\text{train}}+ m_{\text{test}}]\}$ be the data set (e.g. ImageNet, CIFAR10.) Here $x_i\in\mathbb R^{3d}$ is a flattened input image, and $m_{\text{train}}$ and $m_{\text{test}}$ are the numbers of images in the training set (including data-augmentation) and in the test set, resp.
\paragraph{$\Odthree$ Transformation.} Randomly sample $Q\in \Odthree$, a $3d\times 3d$ orthogonal matrix.  The transformed dataset $\tau D$ is 
\begin{align}
    \tau D := \{(Qx_i, y_i): i \in [m_{\text{train}}+ m_{\text{test}}] \}.
\end{align}
The $P(3d)$ (permutation) transformation is defined similarly. 
\paragraph{$\Othreed$ Transformation.} We reshape each $x_i\in\mathbb R^{3d}$ into $x_i\in\mathbb R^{H\times W\times 3}$ where $H$ and $W$ are the height and width of the images (e.g. $H=W=32$, $d=32^2$ for CIFAR10 and $H=W=224$, $d=224^2$ for ImageNet). Independently sample $HW$ many $3\times 3$ random orthogonal matrix $(Q_{h,w})_{h\in [H], w\in [W]}$. The transformed dataset $\tau D$ is
\begin{align}
    \tau D := \left\{\left(\left(Q_{h, w}x_{i; h, w}\right)_{h\in [H], w\in [W]}, y_i\right): i \in [m_{\text{train}}+ m_{\text{test}}] \right\}.
\end{align}
\paragraph{$\Othree$ Transformation.} In this case, we sample only {\it one} $3\times 3$ orthogonal matrix and the transformed dataset is defined to be 
\begin{align}
    \tau D := \left\{\left(\left(Qx_{i; h, w}\right)_{h\in [H], w\in [W]}, y_i\right): i \in [m_{\text{train}}+ m_{\text{test}}] \right\}.
\end{align}
Note that the same rotation matrix $Q$ is applied to all pixels regardless of their spatial locations. 
\paragraph{}

\section{Proof of Theorem \ref{thm:change of prior}} 
\label{sec:proof of inclusion thm}
We use ${\mathsf {FCN} }_{n}$ to denote the class of functions that can be expressed by $L$-hidden layer fully-connected networks whose widths are equal to $n$. Similar notation applies to other architectures.     
\begin{corollary} We have the following  
\begin{align}
\gapn \subseteq {\mathsf { VEC} }_n \subseteq  {\mathsf {LCN} }_{n}
\subseteq {\mathsf {VEC} }_{dn},\quad  {\mathsf {LCN} }_{n} \subseteq 
{\mathsf {FCN} }_{dn} 
\end{align}
\end{corollary}
\begin{proof}
We only need to prove ${\mathsf {LCN} }_{n}
\subseteq {\mathsf {VEC} }_{dn}$ because the others are obvious. 
 Let $\lcnn(x)_{\alpha, i}^l$ denote the post-activation at layer $l$, spatial location $\alpha$ and channel index $i$ of a $\lcnn$ with input $x$ and $\vecn(x)_{\alpha, i}^l$ is defined similarly. It suffices to prove that for any LCN with width $n$ there is a VEC with width $dn$ such that for any $l\geq 1$ (i.e. not the input layer)
\begin{align}
\label{eq:lcn-vec-induction2}
{\mathsf {VEC} }_{dn}(x)_{\alpha, \alpha n + i}^l    = \lcnn(x)_{\alpha, i}^l
\end{align}
since we could choose the readout weights of ${\mathsf {VEC} }_{dn}$ at locations $(\alpha, \alpha n + i)$ to match the one of $\lcnn$ at locaton $(\alpha, i)$ and zero out the remaining entries. 
We prove this by induction and assume it holds for $l$ (the base case $l=1$ is obvious). Then the $\lcnn$ and $\vecn$ at layer $l+1$ can be written as 
\begin{align*}
\lcnn(x)^{l+1}_{\alpha, j} = \phi\left( \frac 1 {\sqrt{n(2k+1)}}\sum_{i\in [n], \beta\in[-k, k]}  \lcnn(x)^{l}_{\alpha +\beta, i} \omega^{l+1}_{\beta, ij}(\alpha) \right) 
\end{align*}
and 
\begin{align*}
{\mathsf {VEC} }_{dn} (x)^{l+1}_{\alpha, j} = \phi\left( \frac 1 {\sqrt{dn(2k+1)}}\sum_{i\in [dn], \beta\in[-k, k]}  {\mathsf {VEC} }_{dn} (x)^{l}_{\alpha+\beta, i} \tilde \omega^{l+1}_{\beta, ij} \right) 
\end{align*}
One can show that \eqref{eq:lcn-vec-induction2} holds for $(l+1)$ by choosing the parameters of ${\mathsf {VEC} }_{dn} $ as follows
\begin{align*}
    \tilde \omega^{l+1}_{\beta, ij}   = \sqrt d \omega^{l+1}_{\beta, i -(\alpha+\beta)n, j -\alpha n}
    \quad \text{if} \quad \alpha n \leq j < \alpha(n+1) \quad \text{and} \quad (\alpha + \beta) n \leq i < (\alpha + \beta) (n+1)
\end{align*}
and $0$ otherwise.

\end{proof}
\section{Proof of Symmetries}
\label{sec:proof of symmetries}
\begin{proof}
    For simplicity, we present the proof for full-batch training. The proof can be applied to mini-batch training as long as the order of the mini-batch is fixed. 
    Let $\tau$ be a rotation in $\Odthree$ or $\Othreed$ or $\Othree$, depending on the architectures ($\fcnn,\lcnn, \vecn, \gapn$) and the tuple $\theta$ and $\gamma$ denote the parameters of the first and remaining layers of the network, respectively. Let $h(\tau x, \theta) = \langle \tau x, \theta\rangle$ denote the pre-activations of the first-hidden layer in the rotated coordinate. Here $\langle \cdot , \cdot \rangle $ is the bilinear map (a dense layer or a convolutional layer with or without weight-sharing, etc.), not the inner product. The loss with $L^2$-regularization is 
    \begin{align}
        R_{\lambda}(\theta, \gamma) &= L(h(\tau \X, \theta), \gamma) + \frac 1 2 \lambda (\left \|\theta\|^2_2 +\|\gamma\|_2^2\right) 
    \end{align}
    where 
    $L(h(\tau \X, \theta), \gamma)$ is the raw loss of the network. For each random instantiation $\theta=\theta_0$ with $\theta_0$ drawn from standard Gaussian iid, 
    we instantiate a coupled network from the un-rotated coordinates but with a different instantiation in the first layer $\theta^{\tau} = \tau^* \theta_0$ and keep the remaining layers unchanged, i.e. $\gamma^{\tau} = \gamma_0$. Here $\tau^*$ is the adjoint of $\tau$ and note that $\tau^* \theta_0$ and $\theta_0$ have the same distribution by the Gaussian initialization of $\theta_0$ and the definition of $\tau$. The regularized loss associated to this instantiation is 
    \begin{align}
                R_{\lambda}(\theta^{\tau}, \gamma^{\tau}) &= L(h(\X, \theta^\tau), \gamma^\tau) + \frac 1 2 \lambda( \left \|\theta^\tau\|^2_2 +\|\gamma^\tau\|_2^2\right) 
    \end{align}
    It suffices to prove that for each instantiation $\theta=\theta_0$ drawn from Gaussian, the following holds for all gradient steps $t$
    \begin{align}
        (\theta^{\tau}_t, \gamma^{\tau}_t) = (\tau^{*} \theta_t,  \gamma_t).  
    \end{align}
    We prove this by induction on $t$ and $t=0$ is true by definition. Assume it holds when $t=t$. Now the update in $\gamma$ and $\gamma^{\tau}$ with learning rate $\eta$ are  
    \begin{align}
        \gamma_{t+1} = \gamma_t -\eta \left(\left.\frac {\partial L}{\partial \gamma}\right\vert_{(h(\tau \X, \theta_t), \gamma_t)}\right)^T  -\eta \lambda \gamma_t 
        \\
         \gamma_{t+1}^{\tau} = \gamma_t^{\tau} -\eta \left(\left.\frac {\partial L}{\partial \gamma}\right\vert_{(h(\X, \theta_t^{\tau}), \gamma_t^{\tau})}\right)^T - \eta\lambda \gamma_t^{\tau}
    \end{align}
    It is clear $ \gamma_{t+1} =  \gamma_{t+1}^{\tau} $ by induction since $h(\tau \X, \theta_t) = h(\X, \theta^{\tau}_t)$. Similarly, 
     \begin{align}
        \theta_{t+1} = \theta_t -\eta \left(\frac {\partial L}{\partial h}
        \left.\frac {\partial h}{\partial \theta}\right\vert_{ (\tau \X, \theta_t))}\right) ^T - \lambda \theta_t 
        \\
         \theta_{t+1}^{\tau} = \theta^{\tau}_t -\eta \left(\frac {\partial L}{\partial h}
         \left.\frac {\partial h} {\partial \theta^{\tau}}\right\vert_{(\X, \theta_t^{\tau})}\right)^T - \lambda \theta_t^{\tau}
    \end{align}
    Note that by the chain rule and induction assumption 
    \begin{align}
        \left.\frac {\partial h} {\partial \theta^{\tau}}\right\vert_{(\X, \theta_t^{\tau}) }
         =  \left.\frac {\partial h}{\partial \theta}\right\vert_{(\X, \theta_t^{\tau})}\frac {\partial \theta^{\tau}}{ \partial \theta}
         =
        \left.\frac {\partial h}{\partial \theta}\right\vert_{(\X, \theta_t^{\tau})}\tau 
    \end{align}
    This implies $\theta_{t+1}^{\tau} = \tau^* \theta_{t+1}$. 
    
    \end{proof} 
    \paragraph{Remark S1.} It is not difficult to see that the proof applies to the Non-Gaussian i.i.d. initialization (e.g. uniform distribution) and/or adding $L^p$-regularization when the rotation groups are replaced by the corresponding permutation groups. Empirically, we observe that replacing the first layer Gaussian initialization by uniform distribution does not change the performance of the network much. See Fig.\ref{fig:uniform vs gaussian.}. 
    
    \paragraph{Remark S2.} The proof works for other parameterization methods, including NTK-parameterization\citep{jacot2018neural}, standard parameterization \citep{sohl2020infinite}, mean-field parameterization\citep{song2018mean} and ABC-parameterization \citep{yang2020feature}

\section{Measuring the Effect of Symmetry Breaking of $\vecn$.}\label{sec:measureing-ssb}

The discussion in the main text suggests that breaking the $\Othreed$ symmetry, making the network to exploit the smaller symmetry group $\Othree$ might be important to good performance of $\vecn$. To measure the effect of symmetry breaking and the reliance of $\vecn$ on the $\Othree$ symmetry, we compare the distance of $\vecn$ to $\vecinf$ ($\Othreed$ invariant) and to $\gapn$ ($\Othree$ invariant). More precisely, for two learning algorithm $\mathcal A_1$ and $\mathcal A_2$ trained on $\D_T$, we defined the square distance between them to be 
\begin{align}
    \text{S-Dist}(\mathcal A_1, \mathcal A_2) = \mathbb E_{x\sim \X}|\mathcal A_1(\D_T)(x) - \mathcal A_2(\D_T)(x)|^2 \, .
\end{align}
If the learning algorithm $\mathcal A_i$ is stochastic, then we use the mean prediction in the above definition. E.g, if $\mathcal A_1$ depends on the initialization $\theta_0$ which is a random variable and $\mathcal A_2$ is deterministic, then we define the squared distance to be  
\begin{align}
    \text{S-Dist}(\mathcal A_1, \mathcal A_2) = \mathbb E_{x\sim \X}|\E_{\theta_0}\mathcal A_1(\D_T; \theta_0)(x) - \mathcal A_2(\D_T)(x)|^2 \, .
\end{align}
Using this definition, we can measure the discrepancy between two rotated systems $(\tau_1\D, \M, \I)$ and $(\tau_2\D, \M, \I)$ by computing $\text{S-Dist}(\mathcal A^{\tau_1}, \mathcal A^{\tau_2})$, where $\mathcal A= (\M, \I)$ and $\tau_{1/2}$ are coordinate transformations. Note that if the system is strictly $\mathcal G$ invariant, then $\text{S-Dist}(\mathcal A^{\tau_1}, \mathcal A^{\tau_2})=0$ for all $\tau_{1/2}\in\mathcal G$. 

We use the exponential map to construct a continuous path\footnote{More precisely, the path lies in $\text {SO}(3)^d\subseteq \Othreed$.} from the identity operator $\Id$ to a random element in $\Othreed$.
More precisely, we randomly sample $d$ $3\times 3$ skew-symmetric matrices $ A = (A_0, \dots, A_{d-1}) \subseteq (\mathbb R^{3\times 3})^d$ and define $\tau = \text {exp}(-{ A})$ and $\tau_t =  \text {exp}(-{t A})$ for $t\in [0, 1]$. Then $(\tau_t)_{t\in[0, 1]}\subseteq \Othreed$ is a continuous path from $\Id$ to $\tau$. We then construct new datasets $\tau_t \D$ (see Fig.~\ref{fig:image-interpolation} for a sample of the continuously rotated images) and study the behavior of the corresponding systems 
\begin{align}
    \{(\tau_t \D, \M, \I): t\in [0, 1]\}
\end{align}

We vary the width $n=64$ to $n=512$ dyadically and $t$ from $[0, 1]$ with equal distance and train the networks on CIFAR10 as in the NN+ setting ($\LR\REG$). Finally, we average the predictions of the learned network over 10 random initialization as an approximation of $\E\vecn^{\tau_t}(x)$ and etc. We summarize the observation below.
\begin{itemize}
    \item As $n$ and/or $t$ increases, the (ensemble) test performance decays monotonically (left panel in Fig.\ref{fig:symmetry breaking}). This is because increasing $n$ and/or $t$ discourages $\vecn$ to utilize the smaller symmetry group $\Othree$. 
    \item As a function of $n$ or $t$, 
    $\text{S-Dist}(\vecn^{\tau_t}, \gapn)$ increases monotonically (middle panel) while $\text{S-Dist}(\vecn^{\tau_t}, \vecinf)$ decreases monotonically (right panel). Thus, small $n$ moves the $\vecn$ learner towards the $\gapn$ learner while increasing $n$ and/or the strength of rotation moves it away from the $\gapn$ learner and towards the $\vecinf$ learner.

\end{itemize}

\begin{figure*}[t]
\centering
\includegraphics[width=1.\textwidth]{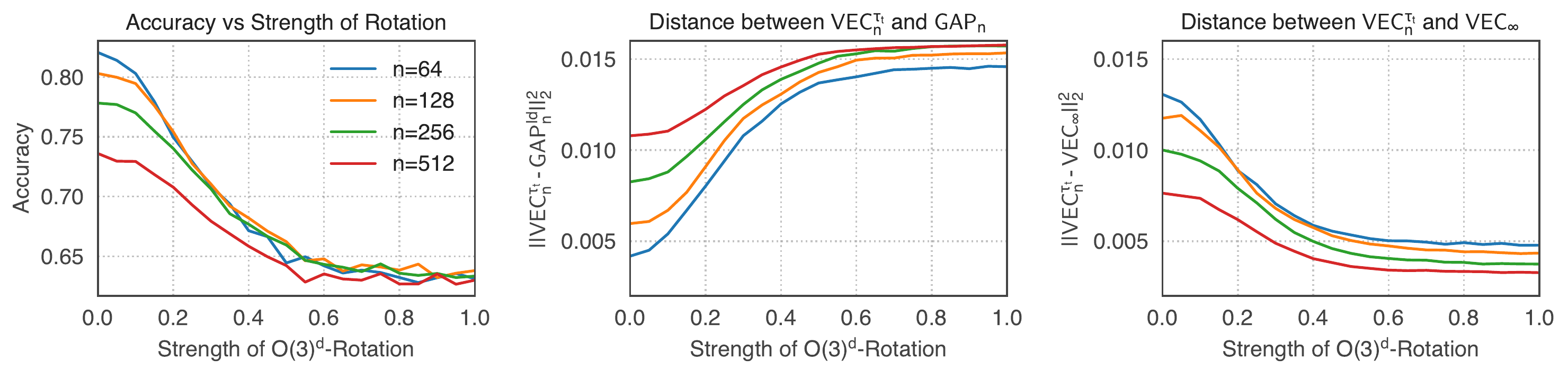}
\caption{{\bf High performing $\vecn$-learner is closer to $\gapn$-learner and far away from $\vecinf$-learner, and vice versa.} {Even in the NN+ setting, $\vecn$ is closer to $\gapn$ for small $n$ and moves towards $\vecinf$ with stronger symmetry and/or larger $n$ and accuracy drops.}}
\label{fig:symmetry breaking}
\end{figure*}    

\begin{figure}
    \centering
    \includegraphics[width=.6\textwidth]{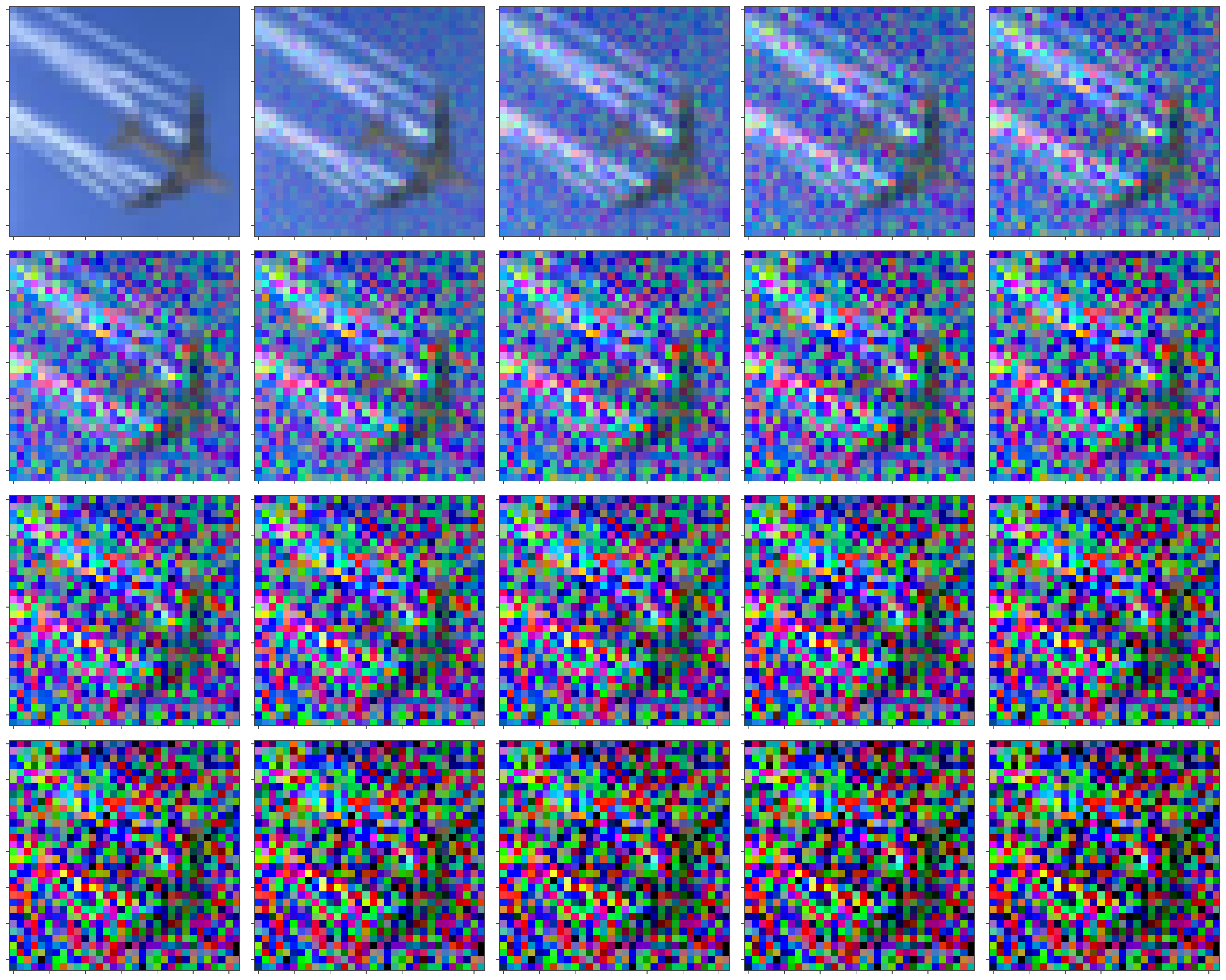}
    \caption{{\bf Continuous deformation from a clean image to the $\Othreed$-rotated image.}} 
    \label{fig:image-interpolation}
\end{figure}

\newpage 
\section{Plots Dump}

\subsection{Scaling Plots for ResNet34 and ResNet101}
\begin{figure}[h]
\centering
          \begin{subfigure}[b]{.32\textwidth}
         \centering
          \includegraphics[width=\columnwidth]{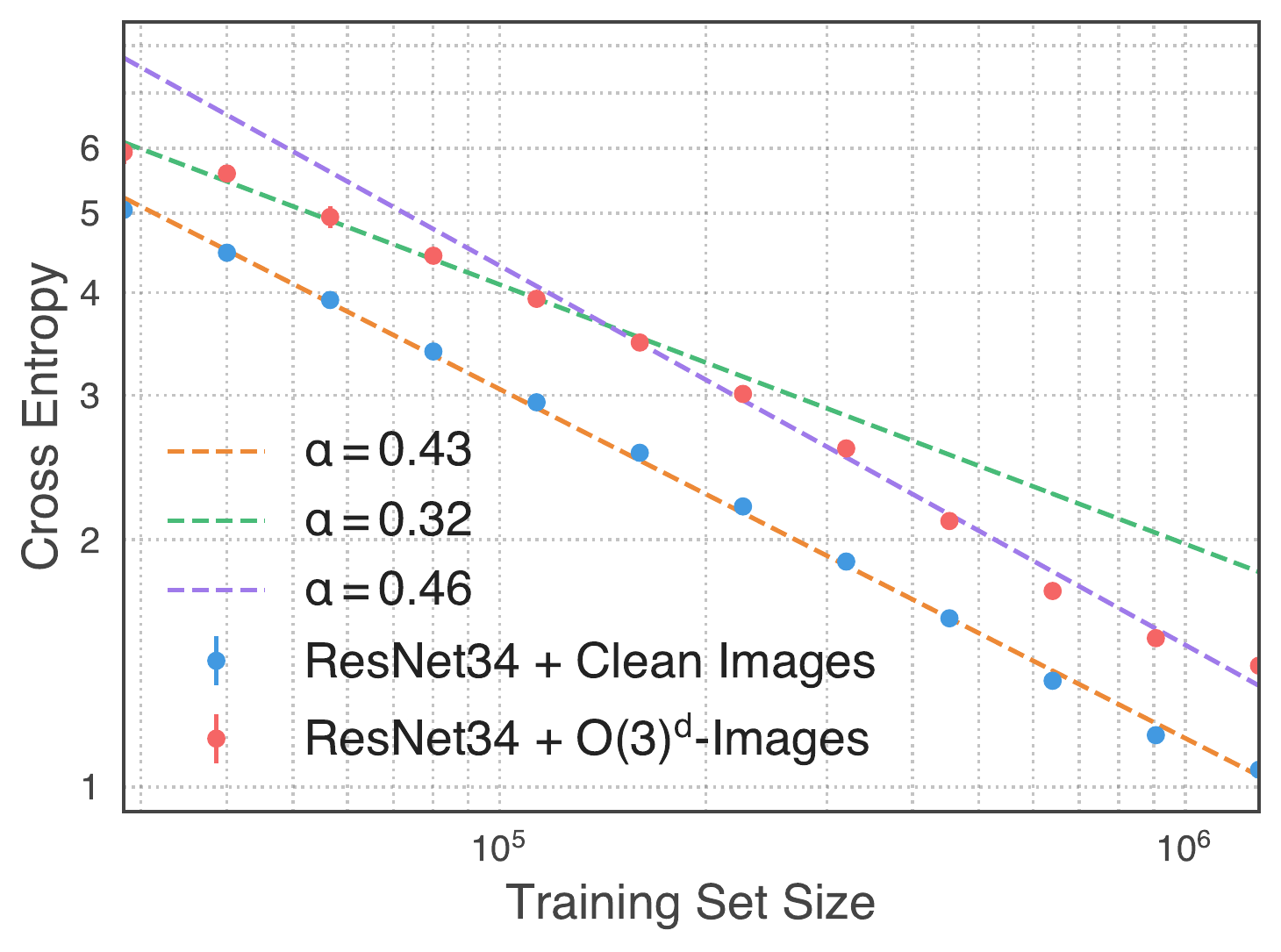}
     \end{subfigure}
              \begin{subfigure}[b]{.32\textwidth}
         \centering
          \includegraphics[width=\columnwidth]{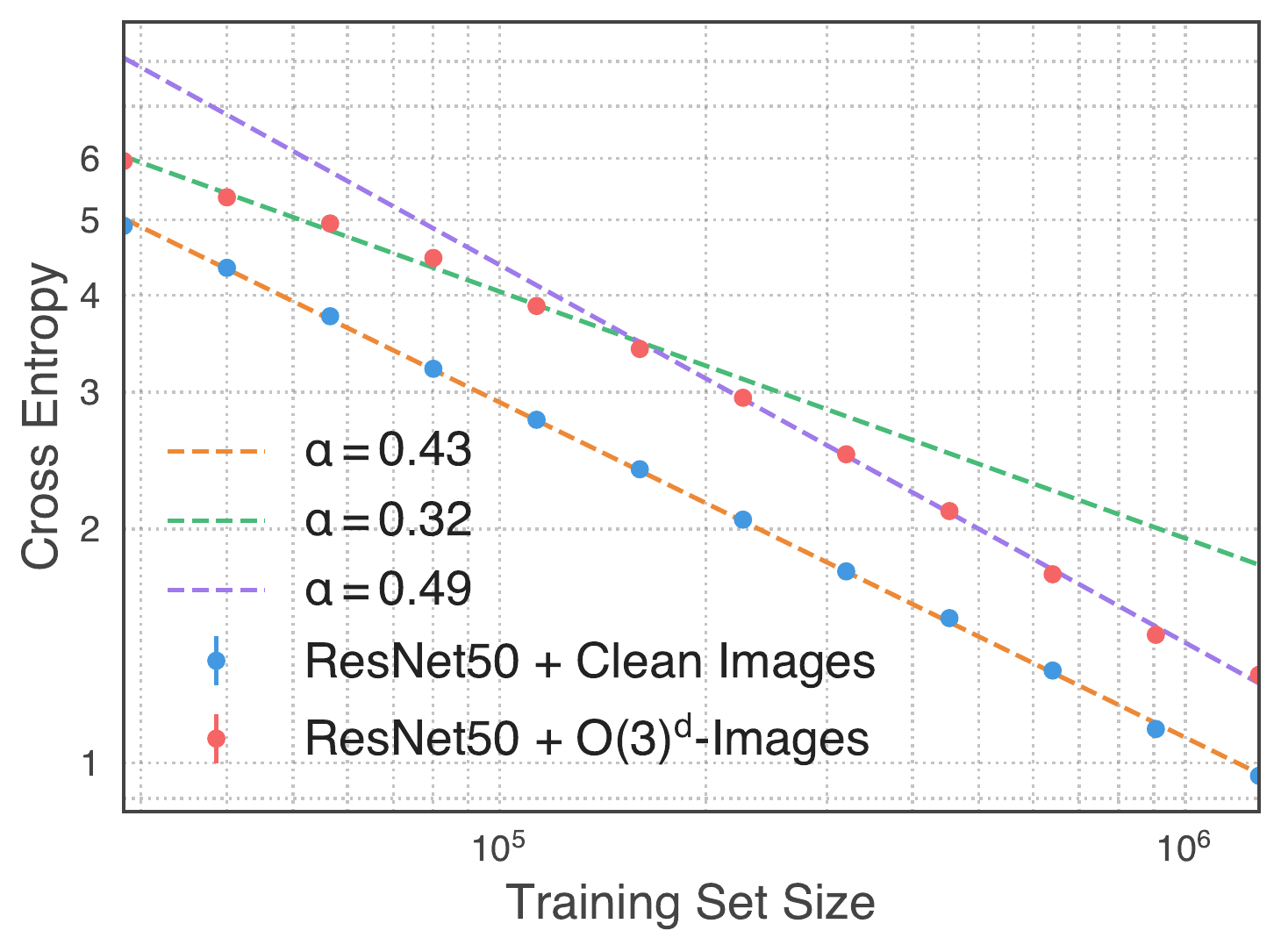}
     \end{subfigure}
     \begin{subfigure}[b]{.32\textwidth}
         \centering
          \includegraphics[width=\columnwidth]{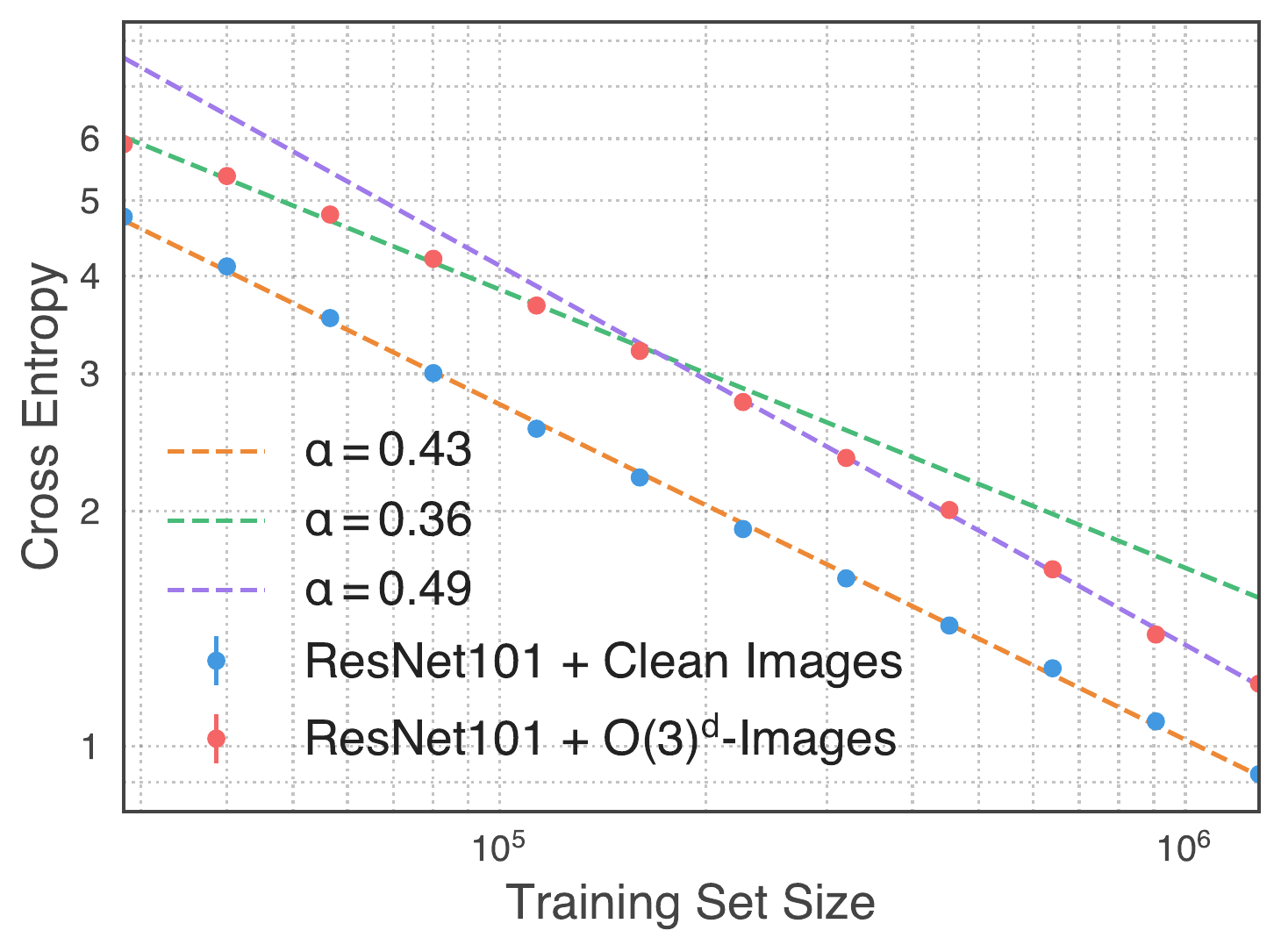}
     \end{subfigure}
 \caption{{\bf Scaling vs Rotation.} Left to right: ResNet34, ResNet50 and ResNet101.} 
    \label{fig:resnet 101}
\end{figure}
\subsection{Learning dynamics of ImageNet for various Rotations}\label{sec:learning-dyanmics-imagenet-rotation}
\begin{figure}[h]
    \centering
     \includegraphics[width=0.3\textwidth]{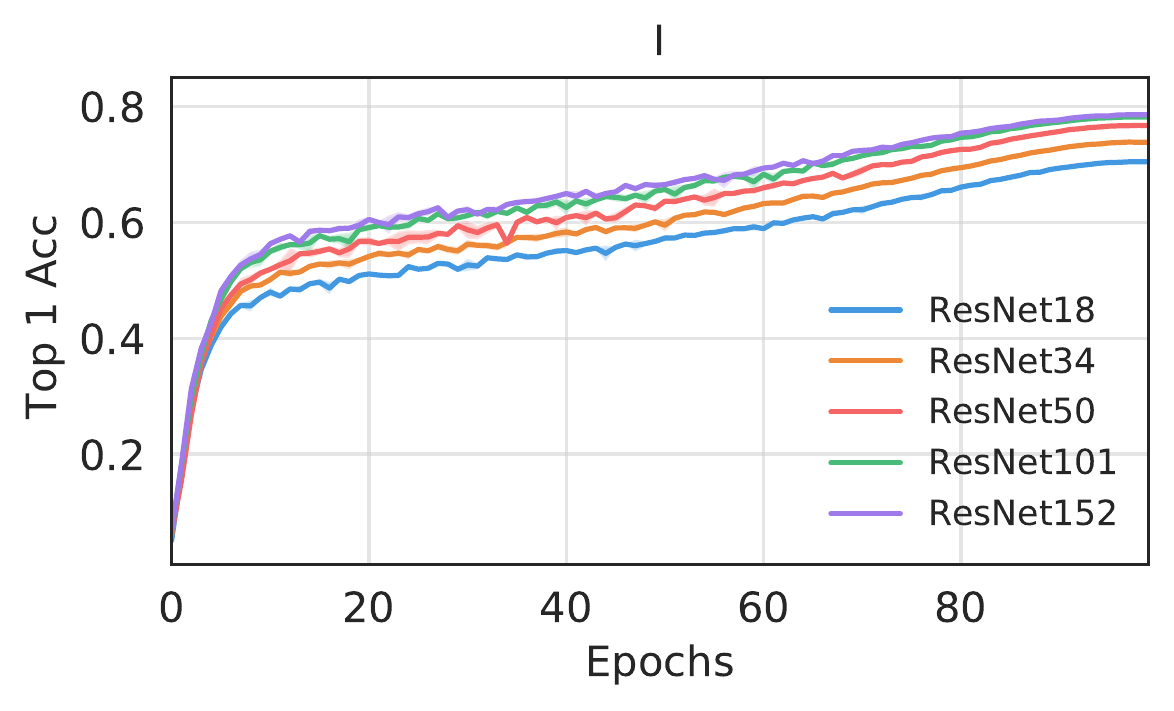}
    \includegraphics[width=0.3\textwidth]{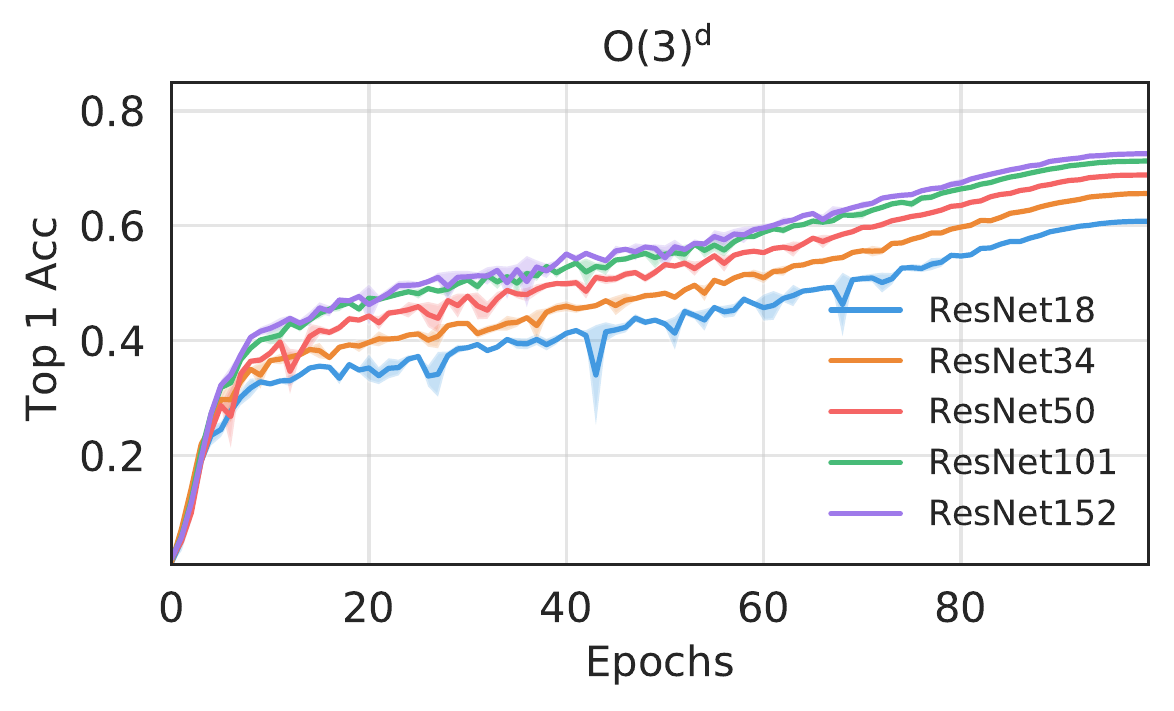}
        \includegraphics[width=0.3\textwidth]{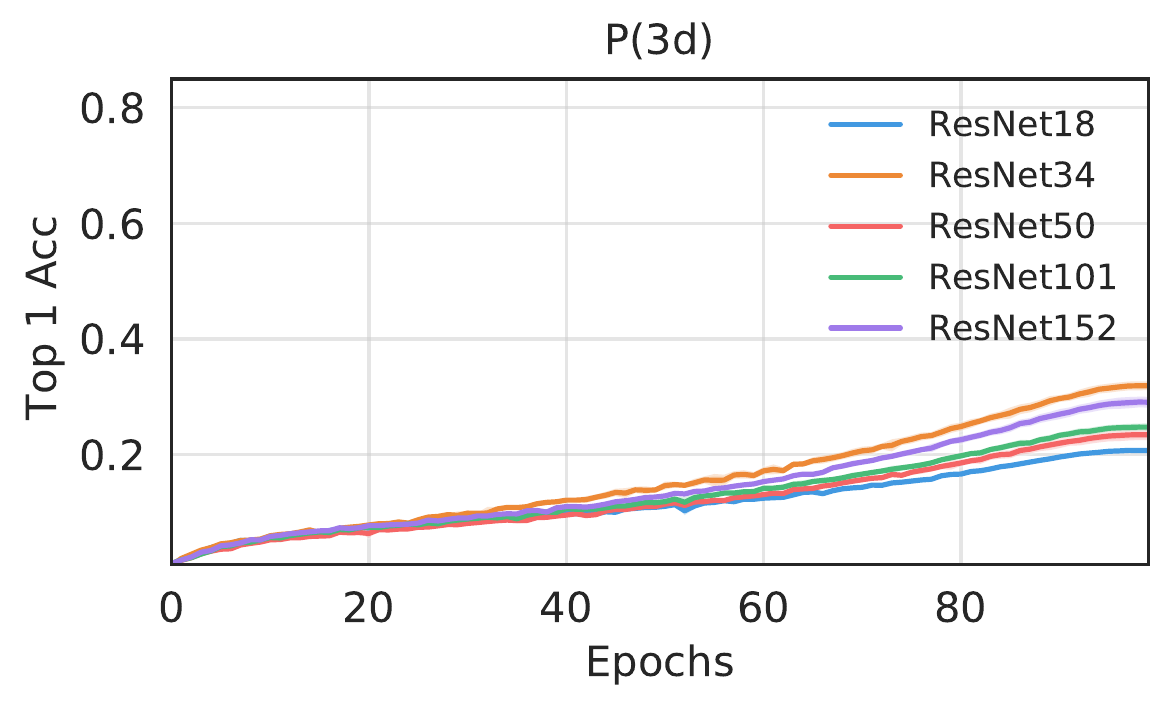}
    \caption{{\bf ResNet can overcome the spurious symmetries $\Othreed$ (middle) but not $\text{P}(3d)$ symmetries (right.)} Learning dynamics (test accuracy vs epochs) of rotated ImageNets. Averaged over 3 runs. Left: no rotation. Middle: $\Othreed$ rotation. Right: $\text{P}(3d)$ rotation.
    When the dataset is $\text{O}(3)^d$-rotated, the models are still able to obtain decent performance, which monotonically improves as the model becomes larger. However, 
    when the dataset is $\text{P}(3d)$-rotated, test accuracy is below 35\% and there isn't a clear trend that larger models can be better. }
    \label{fig:rotation-dynamics}
\end{figure}

\label{sec:plots dump}
\subsection{Gaussian vs Uniform Initializations}
We compare initializing the first layer of the networks using iid Gaussian vs iid Uniform distribution. We observe that the difference is very small; see Fig.~\ref{fig:uniform vs gaussian.}  
    \begin{figure}[h]
    \centering
    \includegraphics[width=.6\textwidth]{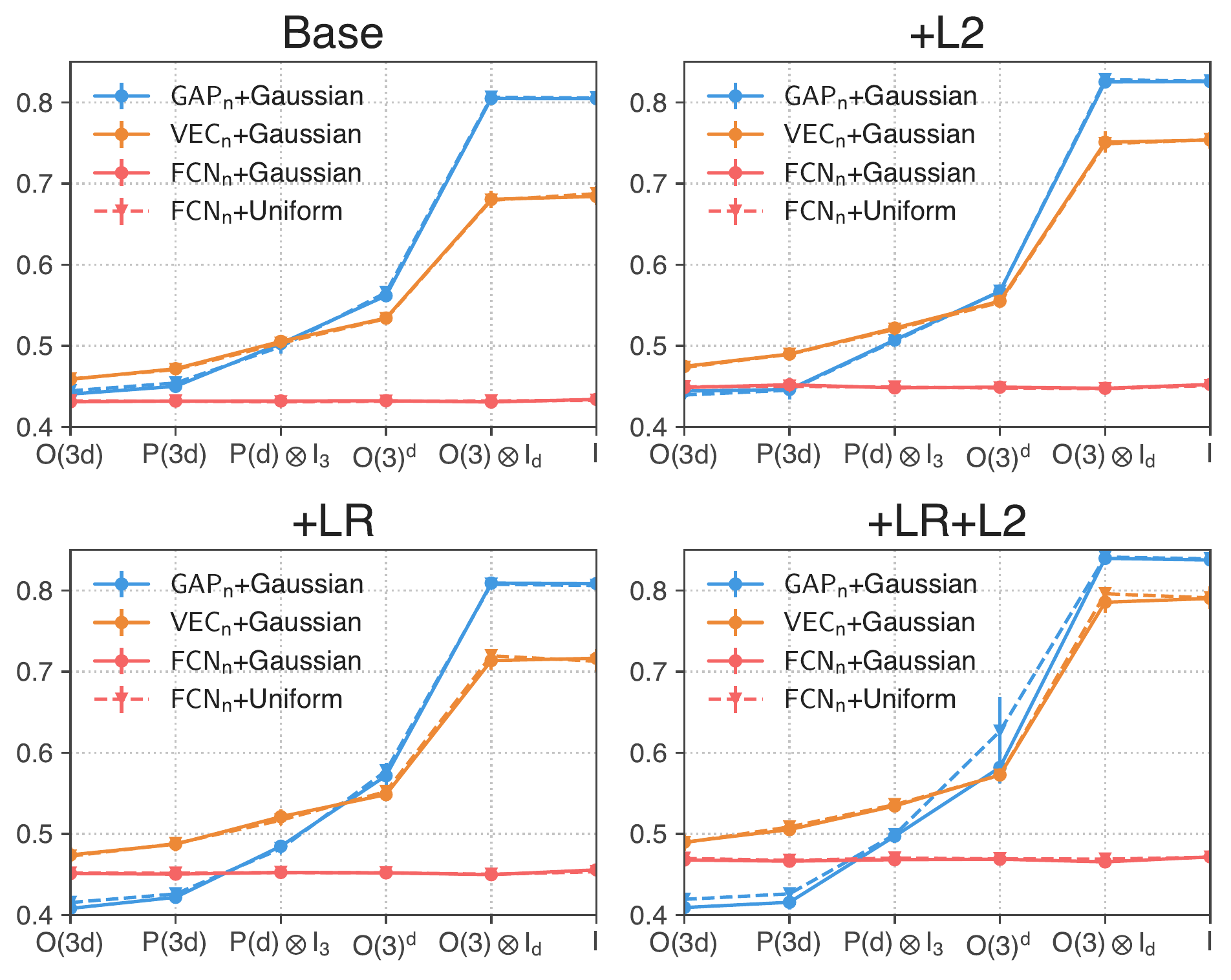}
    \caption{{\bf{Replacing the Gaussian initialization by uniform distribution does not change the performance much.}}}
    \label{fig:uniform vs gaussian.}
\end{figure}

\subsection{Scaling Law for Infinite Networks}\label{sub: scaling nngp ntk}
    \begin{figure}
    \centering
    \includegraphics[width=1.\textwidth]{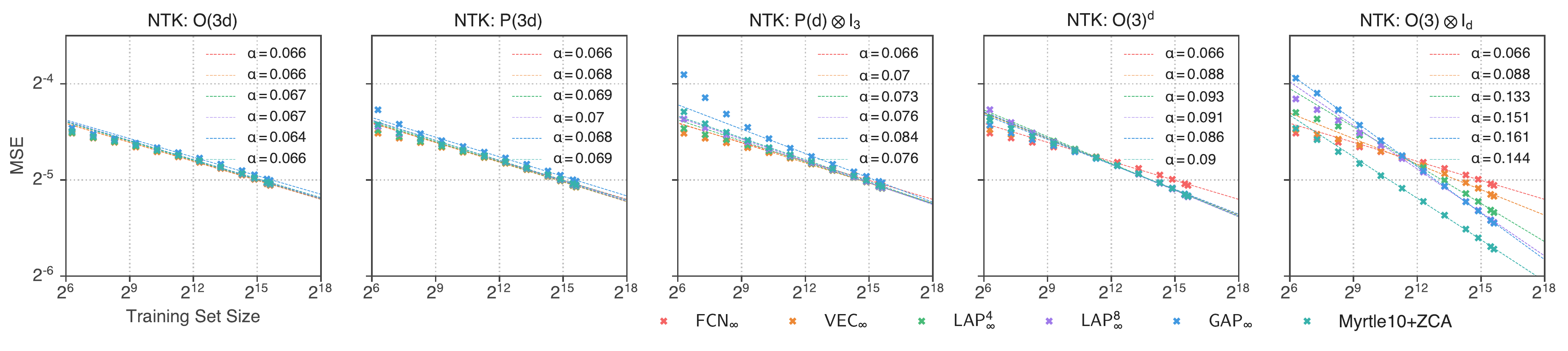}
        \includegraphics[width=1.\textwidth]{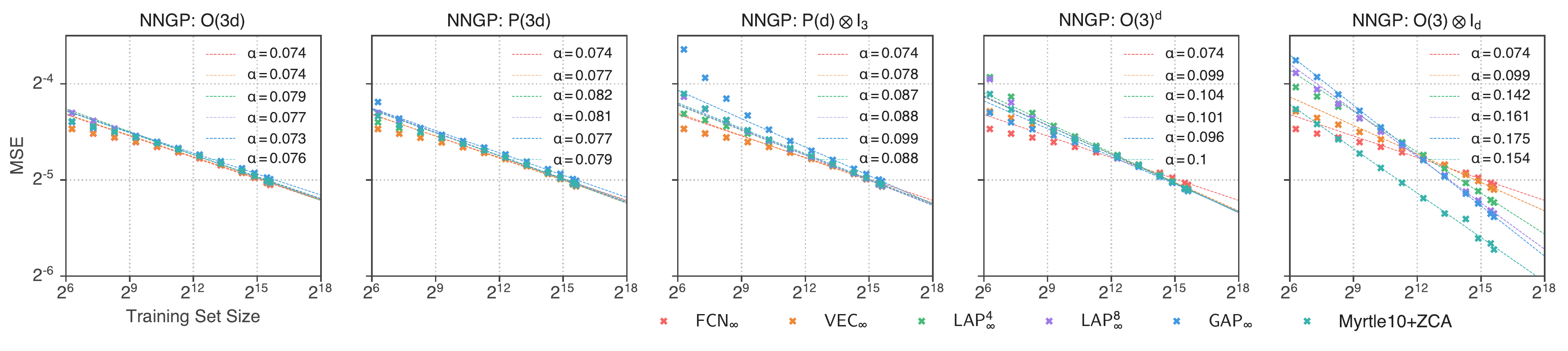}
            \includegraphics[width=1.\textwidth]{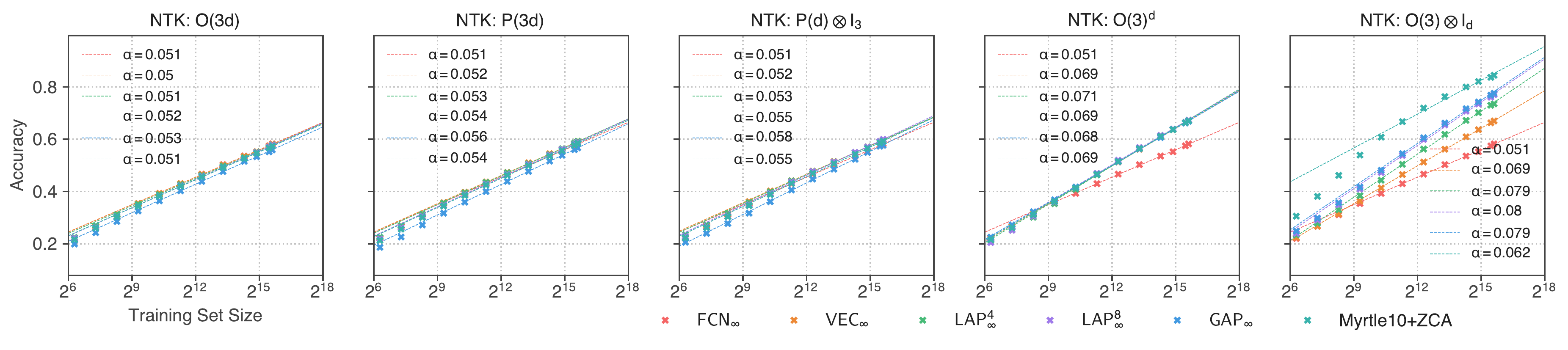}
                \includegraphics[width=1.\textwidth]{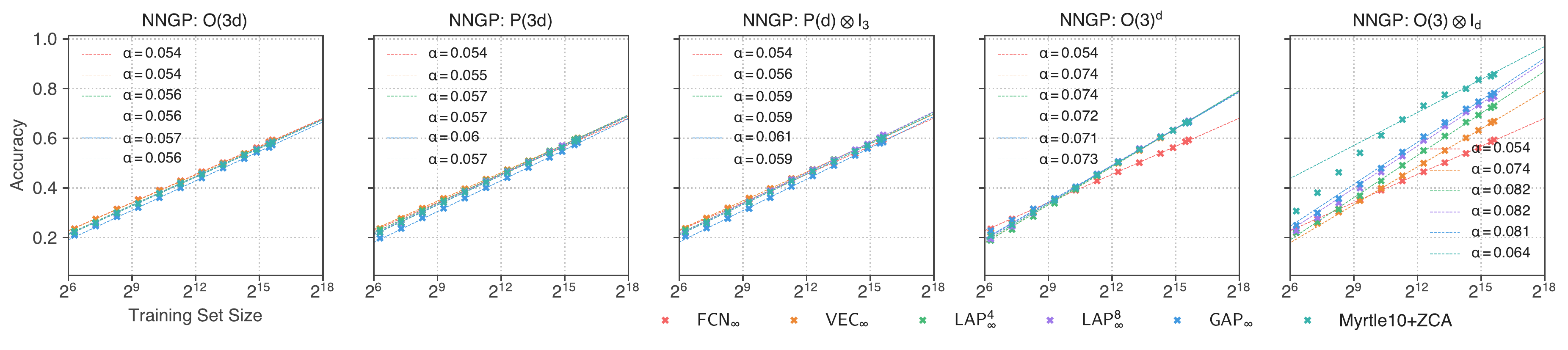}
    \caption{{\bf Scaling Law of Infinite Network vs Different Symmetries vs Architectures.} We see clean power-law lines for most of the learning curves in the MSE plots. The exponents are largely dictated by the spurious symmetries. }
    \label{fig:eigenvalue}
\end{figure}

\subsection{ImageNet Samples}
\begin{figure}[t]
\centering
          \includegraphics[width=\columnwidth]{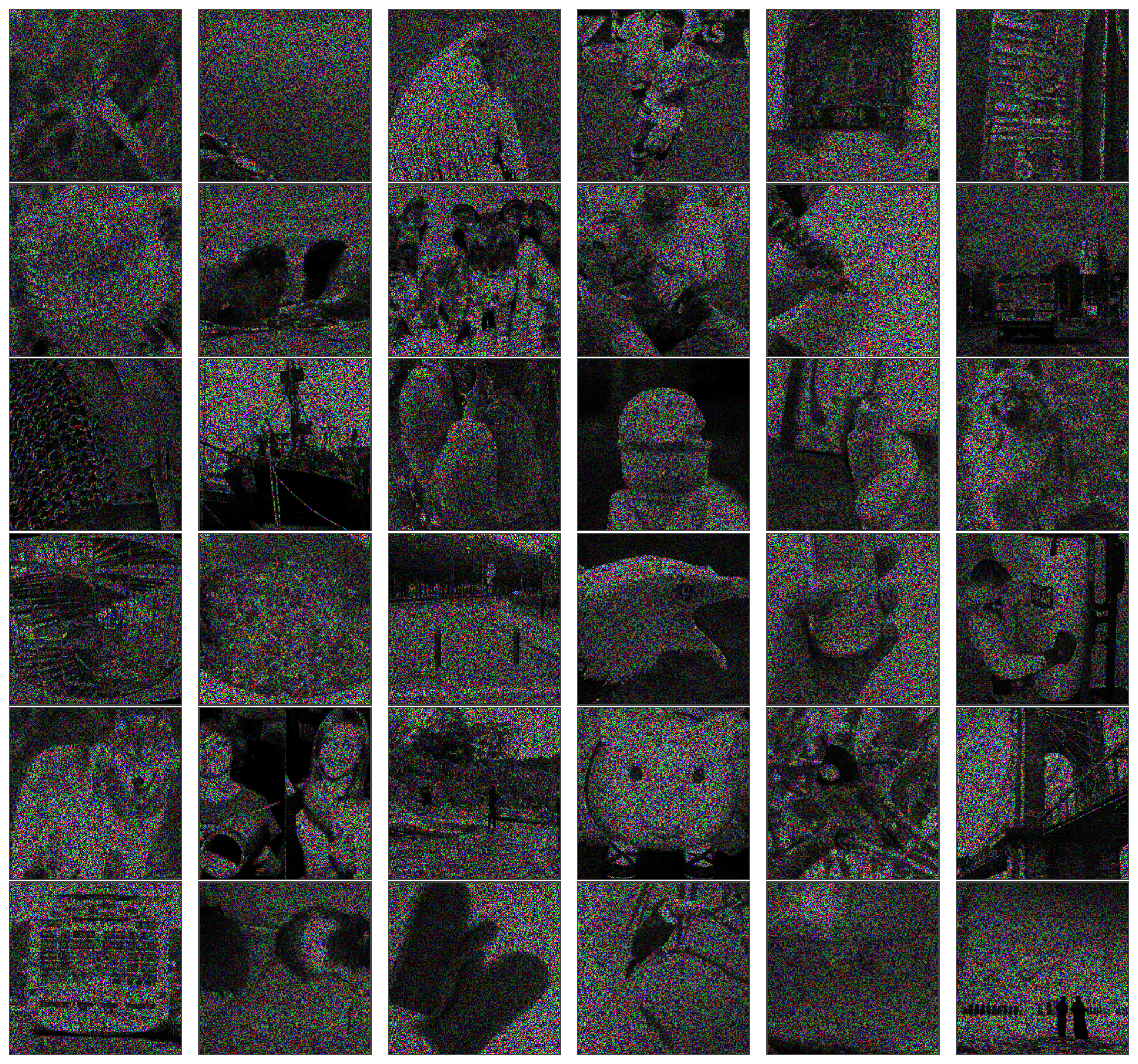}
 \caption{$\Othreed-${\bf Rotated ImageNet Samples. Seed=1}
 } 
    \label{fig:imagenet sample r1}
\end{figure}

\begin{figure}[t]
\centering
          \includegraphics[width=\columnwidth]{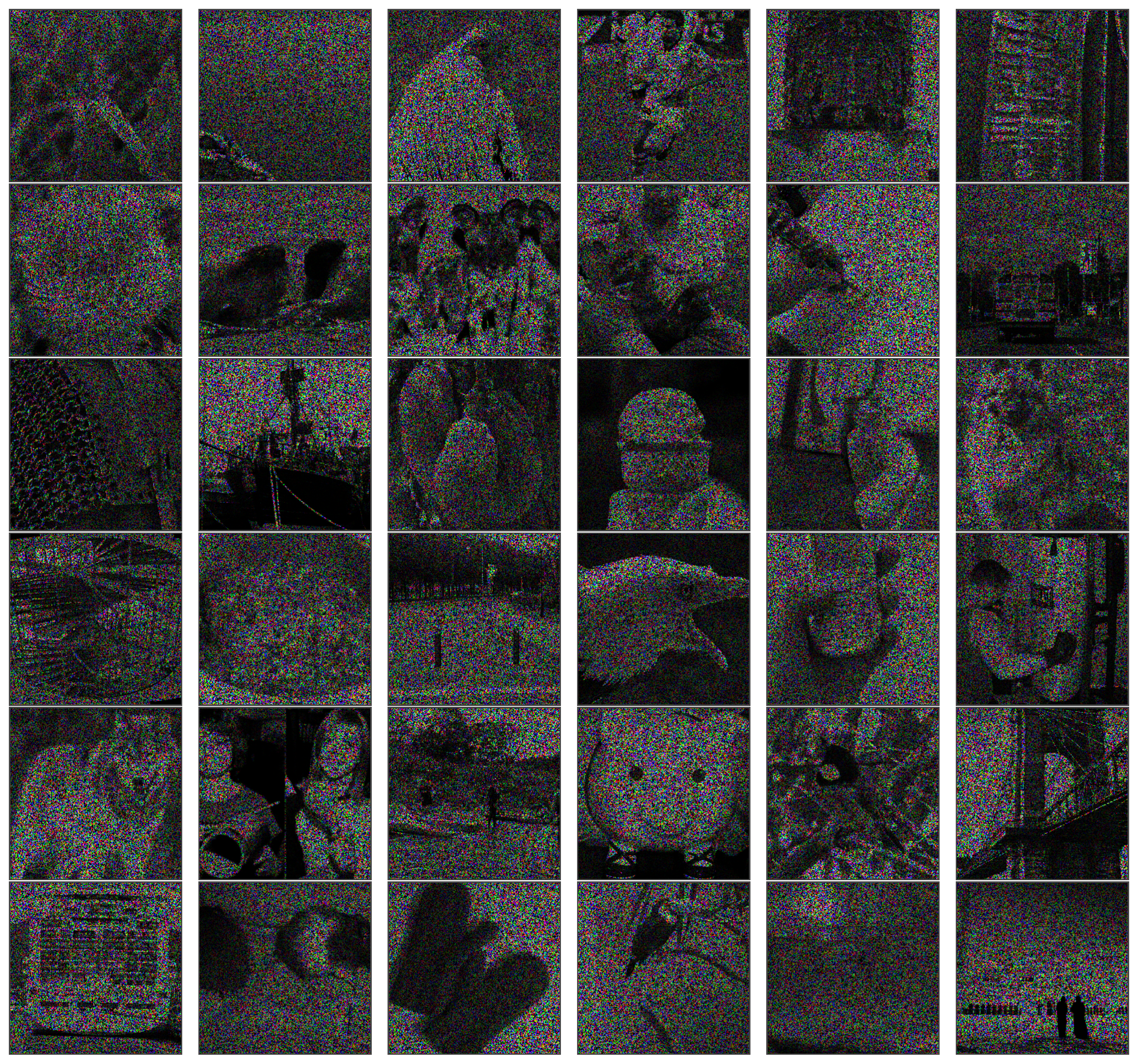}
 \caption{$\Othreed-${\bf Rotated ImageNet Samples. Seed=2} }
    \label{fig:imagenet sample r2}
\end{figure}

\begin{figure}[t]
\centering
          \includegraphics[width=\columnwidth]{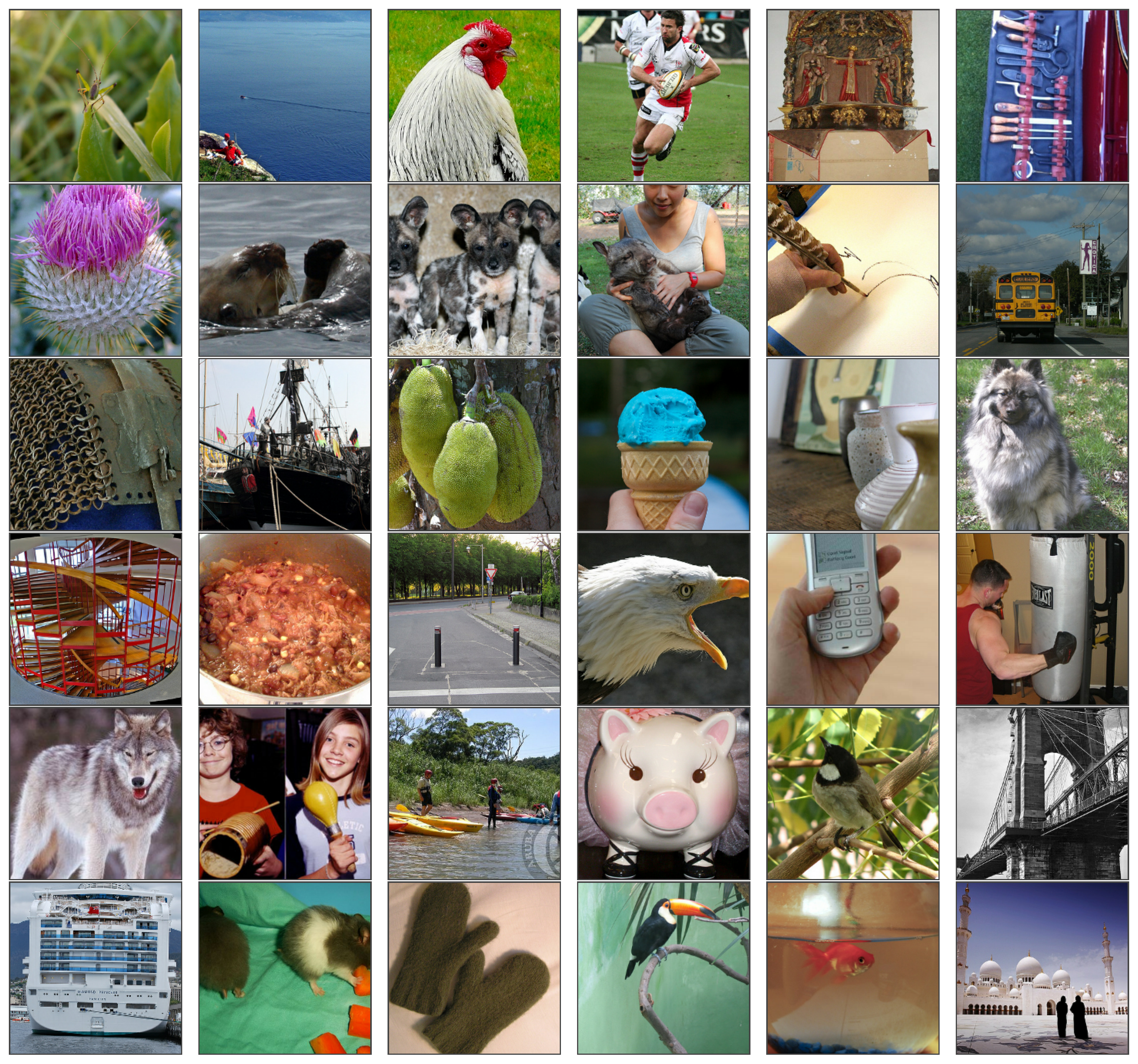}
 \caption{{\bf Clean ImageNet Samples}
 } 
    \label{fig:imagenet sample clean}
\end{figure}


\end{document}